\documentclass{article}
\usepackage{xurl}
\usepackage{microtype}
\usepackage{graphicx}
\usepackage{booktabs} 
\usepackage{hyperref}
\usepackage{url}
\usepackage{wrapfig,lipsum}
\usepackage{pifont}
\usepackage{hyperref}

\usepackage{amsmath,amsfonts,bm}

\def\eqref#1{equation~\ref{#1}}

\def\1{\bm{1}}

\DeclareMathAlphabet{\mathsfit}{\encodingdefault}{\sfdefault}{m}{sl}
\SetMathAlphabet{\mathsfit}{bold}{\encodingdefault}{\sfdefault}{bx}{n}

\newcommand{\E}{\mathbb{E}}

\DeclareMathOperator*{\argmin}{arg\,min}

\usepackage{amsfonts,amsmath,amssymb,amsthm,bm, bbm}
\usepackage{mathrsfs}
\usepackage{subcaption}   
\usepackage{appendix}
\usepackage{authblk}

\def\eqref#1{Equation~(\ref{#1})}

\def\1{\bm{1}}

\RequirePackage{amsmath}
\RequirePackage{amssymb}
\RequirePackage{amsthm}
\RequirePackage{bm} 
\RequirePackage{url}
\usepackage{multirow}
\usepackage{natbib}
\usepackage{graphicx}
\usepackage{makecell}
\usepackage{booktabs}
\usepackage{array}
\usepackage{url}
\usepackage{dsfont}

\usepackage{enumerate}
\usepackage[OT1]{fontenc}
\usepackage{natbib}
\usepackage{mathrsfs}
\usepackage{hyperref}

\usepackage{amsfonts,amsmath,amssymb,amsthm,bm}

\def\##1\#{\begin{align}#1\end{align}}
\def\$#1\${\begin{align*}#1\end{align*}}

\let\tilde\widetilde

\newcommand{\cD}{\mathcal{D}}

\newcommand{\cN}{\mathcal{N}}

\usepackage{xcolor}

\definecolor{red1}{HTML}{f47983}
\definecolor{blue1}{HTML}{3eede7}
\definecolor{yellow1}{HTML}{f5dd6f}

\newtheorem{condition}{Condition}

\newcommand{\tx}{\tilde{x}}
\newcommand{\tz}{\tilde{z}}
\newcommand{\ty}{\tilde{y}}
\newcommand{\ttheta}{\tilde{\theta}}
\newcommand{\TX}{\tilde{X}}
\newcommand{\TY}{\tilde{Y}}
\newcommand{\tep}{\tilde{\epsilon}}
\newcommand{\tsig}{\tilde{\sigma}}

\newcommand{\tvar}{\tilde{\varSigma}}

\newcommand{\tgam}{\gamma}

\usepackage[accepted]{icml2025}

\usepackage{amsmath}
\usepackage{amssymb}
\usepackage{mathtools}
\usepackage{amsthm}

\usepackage[capitalize,noabbrev]{cleveref}

\theoremstyle{plain}
\newtheorem{theorem}{Theorem}[section]

\newtheorem{lemma}[theorem]{Lemma}

\theoremstyle{definition}

\theoremstyle{remark}

\newcommand{\cmark}{\ding{51}}%
\newcommand{\xmark}{\ding{55}}%

\usepackage[textsize=tiny]{todonotes}

\icmltitlerunning{Understanding Overadaptation in Supervised Fine-Tuning: The Role of Ensemble Methods}

\begin{document}

\twocolumn[
\icmltitle{Understanding Overadaptation in Supervised Fine-Tuning:\\ The Role of Ensemble Methods}



\icmlsetsymbol{equal}{*}

\begin{icmlauthorlist}
\icmlauthor{Yifan Hao}{equal,y}
\icmlauthor{Xingyuan Pan}{equal,y}
\icmlauthor{Hanning Zhang}{equal,y}
\icmlauthor{Chenlu Ye}{y}
\icmlauthor{Rui Pan}{y}
\icmlauthor{Tong Zhang}{y}

\end{icmlauthorlist}

\icmlaffiliation{y}{University of Illinois Urbana-Champaign, Illinois}

\icmlcorrespondingauthor{Tong Zhang}{tongzhang@tongzhang-ml.org}

\icmlkeywords{Machine Learning, ICML}

\vskip 0.3in
]



\printAffiliationsAndNotice{\icmlEqualContribution} 

\begin{abstract}

Supervised fine-tuning (SFT) on domain-specific data is the dominant approach for adapting foundation models to specialized tasks. However, it has been observed that SFT models tend to forget knowledge acquired during pretraining. In vision models, ensembling a pretrained model with its fine-tuned counterpart has been shown to mitigate this issue \citep{wortsman2022robust}. In this work, we demonstrate that the same holds for language models, and, more strikingly, we observe an \emph{overadaptation} phenomenon: the ensemble model not only retains general knowledge from the foundation model but also outperforms the fine-tuned model even on the fine-tuning domain itself.
Despite the empirical success of ensembling, a theoretical understanding of its benefits remains underexplored. We develop a formal theoretical analysis of the overadaptation phenomenon.
 Ensembling mitigates this by balancing two primary sources of error: bias, caused by insufficient fine-tuning, and variance, introduced by overfitting to fine-tuning data. While regularization techniques aim to address this trade-off, we show that ensembling provides a more effective solution. We analyze this phenomenon in over-parameterized linear settings and demonstrate that interpolating between pretrained and fine-tuned weights significantly improves performance. These findings offer theoretical justification for the observed advantages of model ensembling, supported by empirical experiments consistent with our analysis.

\end{abstract}

\section{Introduction}



With the remarkable success of large language models (LLMs) such as GPT-4~\citep{achiam2023gpt}, Gemini~\citep{team2023gemini}, and Claude~\citep{anthropic2023introducing}, the pretrain-finetune paradigm has gained significant attention for its outstanding performance. Supervised fine-tuning (SFT) is a widely adopted approach for adapting foundation models to specific downstream tasks. However, a well-known challenge with SFT models is the tendency to forget information acquired during pre-training \citep{mccloskey1989catastrophic, goodfellow2013empirical}.
Model ensembling, also known as model averaging, has emerged as one of the most effective strategies to address this issue. Its benefits have been empirically demonstrated in vision models~\citep{wortsman2022robust} and in reinforcement learning from human feedback (RLHF)~\citep{lin2023speciality}. By simply interpolating the weights of pre-trained and fine-tuned models, ensembling has shown competitive performance in mitigating forgetting compared to other approaches.
In this work, we observe that the same advantage extends to supervised fine-tuning of LLMs. Moreover, beyond its effectiveness in mitigating forgetting on upstream tasks, we also observe a surprising phenomenon of \emph{overadaptation}, which reveals that model ensembling can outperform the fine-tuned model even on downstream tasks, where the fine-tuned model is expected to excel, which also aligns with previous results in \citet{wortsman2022robust, lin2023speciality}.

However, despite the impressive empirical effectiveness of ensemble methods, the corresponding theoretical insights are still limited, especially in the context of modern over-parameterized model. Most theoretical studies on ensembling have focused on traditional under-parameterized settings~\citep{brown2005managing, lin2023spurious, lin2023speciality}, which do not align with the current use of over-parameterized large neural networks. While some works~\citep{allen2020towards, hao2024benefits} have demonstrated the benefits of ensembling independently trained models, these only address improvements in out-of-distribution (OOD) robustness. To the best of our knowledge, no existing work has addressed the central question:

\emph{Why does ensembling achieve such remarkable efficiency, enhancing both generalization on downstream tasks and mitigating forgetting on upstream tasks?}

In this work, we address this question by investigating the role of ensembling in addressing the \emph{overadaptation} phenomenon. During supervised fine-tuning, specialized fine-tuned models overly focus on downstream tasks, leading to a loss of valuable information retained in the pre-trained model. The effectiveness of ensembling can be attributed to its ability to mitigate overadaptation. By simply averaging the weights of the pre-trained and fine-tuned models, the ensemble recovers the information lost during fine-tuning while preserving the knowledge gained from the fine-tuning process.

Specifically, we start with presenting empirical evidence in Section~\ref{seec:empirical} that highlights the harmful effects of overadaptation and demonstrates the efficiency benefits of ensembling in both improving fine-tuning performance and mitigating forgetting. Building on these observations, we develop a formal theoretical framework in Section~\ref{sec:note} and Section~\ref{sec:re} to analyze the effectiveness of model ensembling, focusing on its impact on fine-tuning tasks and pre-training task retention.
We attribute the improved performance of ensembling to its ability to better balance the ``bias-variance'' trade-off in test error. To simplify our explanation, we model the pre-training and fine-tuning processes using an over-parameterized linear setup. Within the context of canonical linear regression, we represent the pre-trained model as the ``ridgeless'' estimator on Task 1. For fine-tuning on Task 2, total overfitting (overadaptation) is characterized by ``ridgeless'' regression, while non-overfitting approaches, such as early stopping, are captured through ridge regression \citep{lin2017optimal, lu2022sobolev}. Our main theoretical findings can be summarized as follows.

\subsection{High-level theoretical insights}
 On the theoretical side, based on an over-parameterized regression setup, after pre-training on Task 1, we fine-tune the model on a specific Task 2. The results are stated on two aspects as follows.

Focusing on the performance on Task 2, we prove that
\begin{enumerate}
\item (Poor Performance of pre-trained Model): The pre-trained model exhibits a high ``bias'' term in test error due to its inability to capture task-specific information for Task 2; 
\item (Limitations of fine-tuning without regularizer): Fine-tuning without any regularization, i.e., using a ``ridgeless'' estimator, results in a high ``variance'' term in test error due to overfitting (overadaptation) on noisy data, leading to poor performance; 
\item (Impact of regularization in fine-tuning): Applying ridge regression during fine-tuning helps mitigate overfitting by achieving a better balance in the ``bias-variance'' trade-off, thereby reducing test error; 
\item (Ensemble for improved trade-off management): Combining the pre-trained model with either the ridge or ``ridgeless'' fine-tuned model enables a more effective balance of the ``bias-variance'' trade-off in test error, leading to improved performance on Task 2.
\end{enumerate}
And for the forgetting phenomenon, we consider the performances on both Task 1 and Task 2, establishing that
\begin{enumerate}
\item (Impact of regularization in forgetting): Without hurting the performance on Task 2, ridge regression mitigates forgetting on Task 1 by balancing the ``bias-variance'' trade-off, effectively managing both pre-trained and fine-tuned errors simultaneously.
\item (Ensemble for enhanced forgetting mitigation): Leveraging the pre-trained and fine-tuned models, ensembling further improves such ``bias-variance'' balance, providing an additional reduction in forgetting.
\end{enumerate}

\subsection{Empirical validation overview}
Our theoretical analysis is mainly inspired by the ``magical'' empirical results, which are deferred to Section~\ref{seec:empirical}. Here we highlight the main results, to show the consistency between our theoretical results and the empirical phenomenon. Specifically, we give empirical evidences showing that:
\begin{enumerate}
\item (Harmful overadaptation in fine-tuning): When the training process extends into the ``overfitting'' regime (without applying early stopping), fine-tuning performance deteriorates;
\item (Enhanced performance through ensemble): Model ensemble consistently improves model performance across various experiment settings, both achieving better performance on fine-tuning tasks and mitigating forgetting on pre-training tasks efficiently.
\end{enumerate}

\section{Related Works}

There exists a substantial body of work on model ensemble, pre-training and fine-tuning. In this section, we review the most relevant works to ours.

\paragraph{Model ensemble.}  Model ensemble has been a popular technique to enhances generalization performance, as documented in the existing literature ~\citep{hansen1990neural, krogh1994neural, perrone1995networks, opitz1999popular, dietterich2000ensemble, zhou2002ensembling, polikar2006ensemble, rokach2010ensemble, rame2022diverse, arpit2022ensemble, kumar2022calibrated, wortsman2022model, lin2023spurious}. Recently, ensemble the pre-trained and fine-tuned models has been verified to benefit the out-of-distribution robustness \citep{wortsman2022robust}, as well as decreasing forgetting in reinforcement learning from human feedback \citep{lin2023speciality}.

\paragraph{Understanding model ensemble.} On the theoretical side, works explaining the good performance of  ensemble are limited, especially in the context of overparameterized models. In traditional underparameterized settings, \citet{brown2005managing} decomposes the prediction error of ensemble models into bias, variance and a covariance term between individual models, proposing algorithms to encourage the diversity of individual models to reduce the covariance term; similarly, \citet{lin2023spurious} showed that ensembling two independently trained models increases feature diversity, thereby improving out-of-distribution (OOD) robustness, while \citet{lin2023speciality} extended this framework to demonstrate that feature diversity also mitigates forgetting on upstream tasks. In the case of over-parameterized models, recent works such as \citet{allen2020towards} and \citet{hao2024benefits} have shown that ensembling independently trained models improves OOD robustness. A comprehensive explanation for the improved model efficiency remains lacking.


\paragraph{Pre-training and fine-tuning.}
The pre-training and fine-tuning paradigm have become a cornerstone in developing high-performance models across various domains, particularly in large language models (LLMs). Recent advancements have focused on enhancing this framework through various techniques.
The LP-FT technique, introduced by \citet{kumar2022fine}, involves initializing the pre-trained feature extractor with a reasonably good classifier; \citet{huang2021continual} proposed low-rank adaptation (LoRA) to reduce the number of trainable parameters during fine-tuning, which benefits parameter-efficient training; \citet{tian2023trainable} presented a trainable projected gradient method aimed at enhancing out-of-distribution (OOD) generalization; model ensemble has also demonstrated effectiveness in improving performance, as evidenced by many studies \citep{cha2021swad, chu2022dna, wortsman2022model, lin2023speciality}.


\section{Motivated by the Intriguing Phenomenon}\label{seec:empirical}

We start our observation on instruction-following fine-tuning tasks, highlighting the empirical findings from three key aspects: (i) the impact of using a regularizer versus overfitting (overadaptation), the performance of ensembling pre-trained and fine-tuned models on (ii) fine-tuning tasks and (iii) pre-training tasks.

\begin{figure}[ht]
\vskip 0.1in
        \centering
        \includegraphics[width=0.48\textwidth]{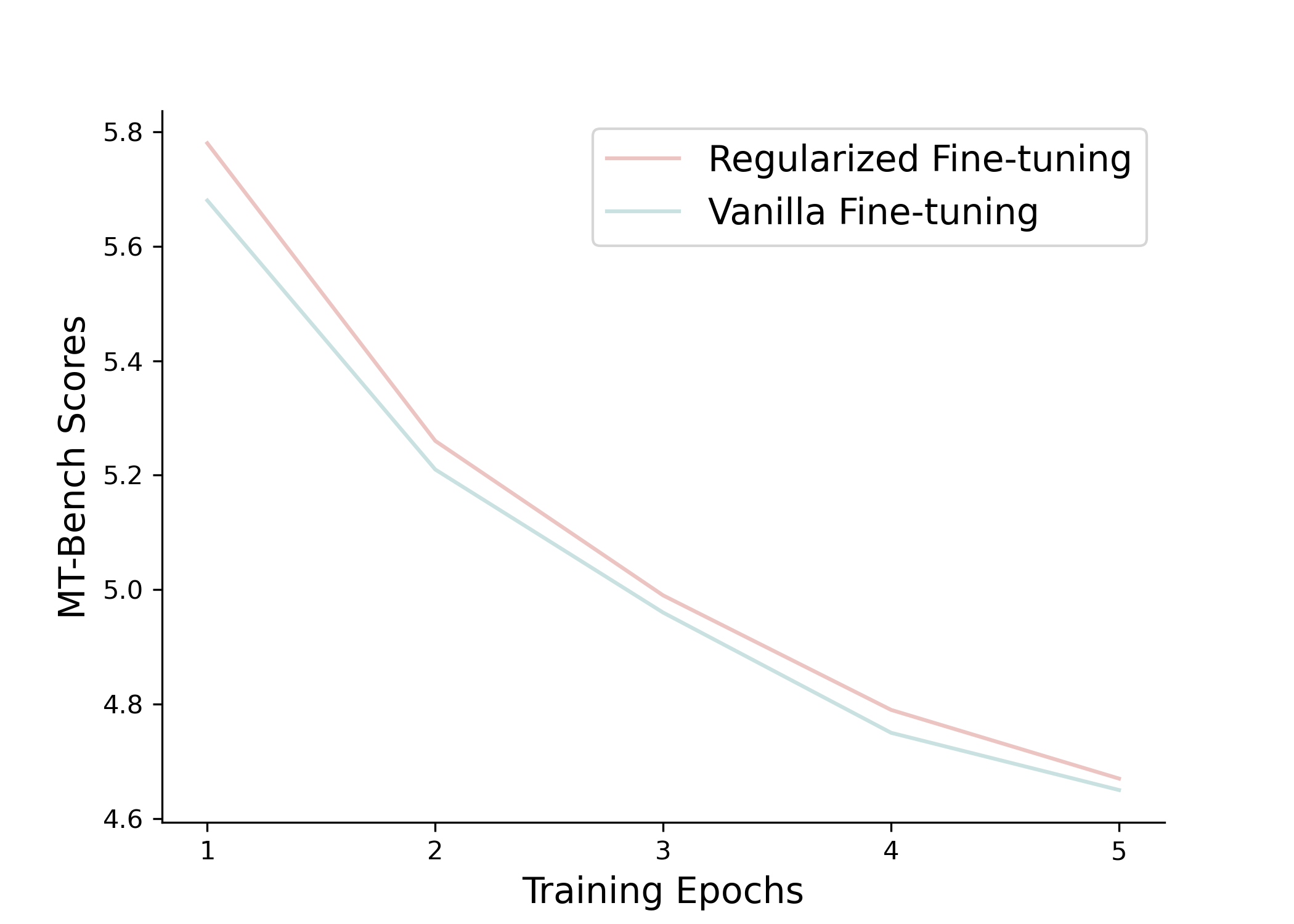}
        \caption{\textbf{Early Stop Experiments:} The performance on MT-Bench when the training epoch increases.
        We conduct vanilla fine-tuning and DiffNorm-Penalty fine-tuning with \textbf{Llama-3-8B} on Dolly dataset.
        }
        \label{fig:epoch}
\vskip 0.1in
\end{figure}

\subsection{Datasets and Benchmarks}

Our experiments utilize the Dolly dataset \citep{DatabricksBlog2023DollyV2}, a popular instruction-following dataset that covers a wide range of tasks, including Creative Writing, Closed QA, Open QA, Summarization, Information Extraction, Classification and Brainstorming, ensuring its high-diversity.


The LLMs’ instruction-following ability is evaluated on MT-Bench \citep{zheng2023judgingllmasajudgemtbenchchatbot} with single-answer grading. This benchmark prompts conversational assistants with challenging multi-turn open-ended questions and utilizes ``LLM-as-a-judge'' for evaluation, which comprises 80 questions, evenly distributed across 8 categories: Writing, Roleplay, Extraction, Reasoning, Math, Coding, Knowledge I (STEM), and Knowledge II (humanities/social science). For each response, the LLM judge of GPT-4 will provide a score on a scale of 1 to 10, indicating the overall instruction-following ability of the evaluated conversational assistant.


We also assess LLMs' general ability on MMLU \citep{hendrycks2021measuringmassivemultitasklanguage} and Commonsense-QA \citep{talmor2019commonsenseqaquestionansweringchallenge}. 
MMLU is a dataset containing 57 tasks including mathematics, chemistry, computer science, law, and more, which measures multi-task ability and requires extensive world knowledge and problem-solving ability.
Commonsense-QA contains more than 10K real-world common sense questions. It requires LLMs to identify related real-world knowledge and distinguish the distracted answers.


The goal of the experiments is to highlight the strengths of regularization and ensembling in improving the tradeoff between instruction following and general abilities, revealing that these common generalization techniques not only enhance downstream task performance, but also alleviate forgetting issues in LLMs.



\begin{figure}[ht]
\vskip 0.1in
        \centering
        \includegraphics[width=0.45\textwidth]{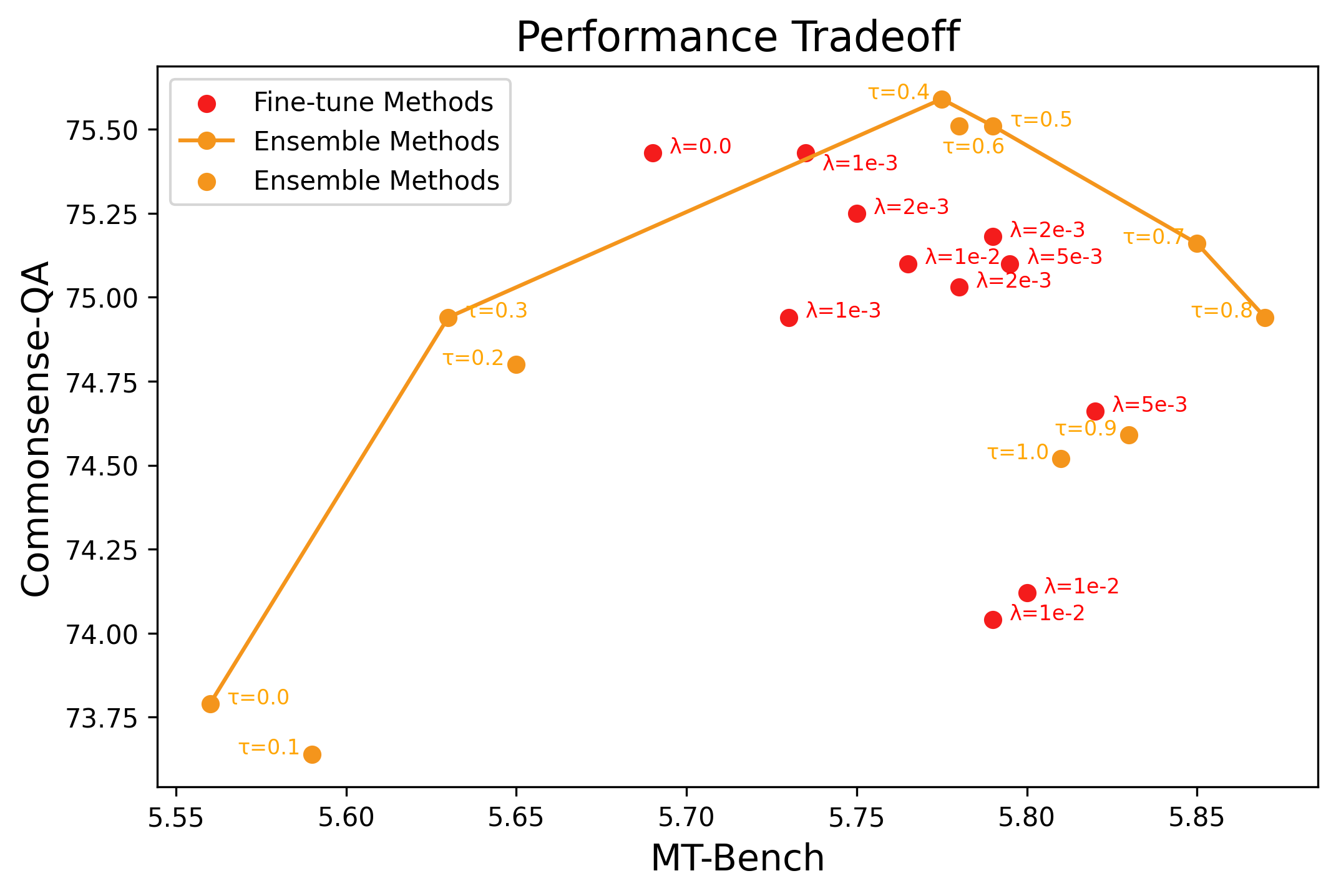}
        \caption{\textbf{Experiments on Commonsense-QA and MT-bench:} The performance tradeoff on MT-Bench and Commonsense-QA 
        for ensemble methods and fine-tune methods. The ensemble methods are based on the pre-trained model and the DiffNorm-Penalty model, where the results of different $\tau$ values and seeds are presented here. We use \textbf{Llama-3-8B} in our experiments.}
        \label{fig:commonsense}
\vskip 0.1in
\end{figure}

\subsection{Experiment Settings}

Our experiments are based on three well-known open-source base models: Llama-3-8B \citep{dubey2024llama3herdmodels}, Qwen2-7B \citep{yang2024qwen2technicalreport}, and Gemma-2-9B \citep{gemmateam2024gemma2improvingopen}, where for each model, we compare the vanilla fine-tuning approach to variants with generalization techniques applied. Depending on whether or not ensembling is used, and which types of regularization are adopted, totaling $2\times 2 = 4$ variants are proposed to be compared with the vanilla fine-tuning baseline.

For regularization, an additional penalty term of $\|\hat{\theta}\|_2^2$ or $\|\hat{\theta} - 
\hat{\theta}_1\|_2^2$ is added to the fine-tuning loss, where $\hat{\theta}$ means the fine-tuned model parameters and $\hat{\theta}_1$ represents the pre-trained weights. We denote those variants as Normal-Penalty and DiffNorm-Penalty separately. For ensembling, the parameters of fine-tuned models are weighted-averaged with the pre-trained model. Those variants are denoted as Avg-Norm-Penalty and Avg-DiffNorm-Penalty, respectively. We make our implementation publicly available \footnote{https://github.com/xypan0/LLMForgetting}. For more experimental details, please refer to Appendix~\ref{appendix:exp_details}.

\subsection{Results}

Our empirical results uncover three key issues in the standard fine-tuning process, as shown in Figure~\ref{fig:epoch}, Table~\ref{MT-Bench}, Figure~\ref{fig:commonsense} and Figure~\ref{fig:mmlu}:

\paragraph{Overfitting (overadaptation) is harmful during the fine-tuning process.} In Figure~\ref{fig:epoch}, we show the training process on the Dolly dataset and performance on MT-bench without applying early-stopping. It becomes evident that performance deteriorates quickly with additional epochs, indicating harmful overfitting. Even with early stopping, as seen in Table~\ref{MT-Bench}, the use of regularizers, such as Norm-Penalty and DiffNorm-Penalty, improves generalization performance compared to non-regularized fine-tuning (Vanilla-FT).

\paragraph{Ensemble enhances generalization performance.} As shown in Table~\ref{MT-Bench}, the ensemble of pre-trained and fine-tuned models consistently outperforms the individually fine-tuned models on fine-tuning tasks. This improved performance highlights the effectiveness of model ensembling, aligning with the empirical findings in \citet{lin2023speciality}. Although both methods, Norm-Penalty and DiffNorm-Penalty, use early-stopping to prevent overfitting, they apply different penalties additionally in the training process. Interestingly, in our experiments, Norm-Penalty consistently outperforms DiffNorm-Penalty in almost all settings, suggesting that the choice of regularizer plays a crucial role and warrants further exploration.

\begin{figure}[ht]
\vskip 0.1in
        \centering
        \includegraphics[width=0.45\textwidth]{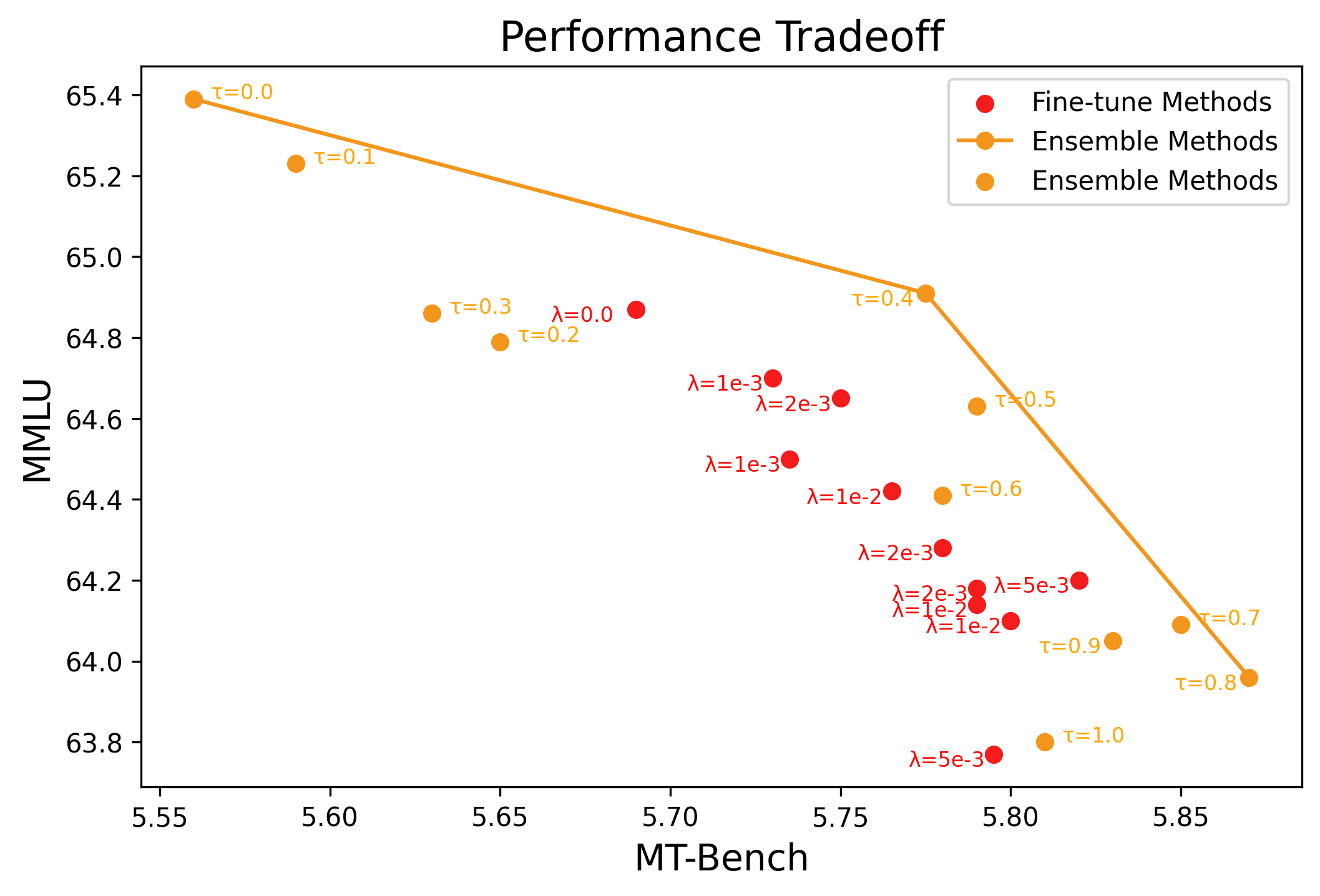}
        \caption{\textbf{Experiments on MMLU and MT-bench:} The performance tradeoff on MT-Bench and MMLU 
        for ensemble methods and fine-tune methods. The ensemble methods are based on the pre-trained model and the DiffNorm-Penalty model, where results of different $\tau$ values and seeds are presented here. We use \textbf{Llama-3-8B} in our experiments.
        }
        \label{fig:mmlu}
\vskip 0.1in
\end{figure}

\paragraph{Ensemble improves trade-off between pre-training and fine-tuning tasks.} We also evaluate the trade-off between pre-training tasks (Commonsense-QA and MMLU) and the downstream task (MT-bench). The results are shown in Figure~\ref{fig:commonsense} and Figure~\ref{fig:mmlu}. Compared to individually fine-tuned models, ensemble models achieve better trade-offs on these tasks especially when $\tau$ is larger than 0.5. The ensemble models generally have a high MT-bench score while maintaining good performance on Commonsense-QA and MMLU. The fine-tuned models would suffer from more forgetting if they achieve high MT-bench scores. The results suggest that ensemble methods could achieve a good balance of instruction-following ability and generalization ability.

\begin{table*}[t]
\caption{MT-Bench scores of models fine-tuned on Dolly~\citep{DatabricksBlog2023DollyV2}.}
\label{MT-Bench}
\vskip 0.15in
\begin{center}
\begin{small}
\begin{sc}
\begin{tabular}{ccc|ccc}
        \toprule
         Methods & Regularizer & Ensembling & Llama-3-8B & Qwen2-7B & Gemma-2-9B \\ \midrule
         Vanilla-FT & - & \xmark             & 5.68 & 6.57 &  6.52 \\ \midrule
         Norm-Penalty & $\|\hat{\theta}\|_2^2$ & \xmark         & 5.84 & 6.81 &  6.59 \\ \midrule
         DiffNorm-Penalty & $\|\hat{\theta} - \hat{\theta}_1\|_2^2$ & \xmark       & 5.78 & 6.68 & 6.65 \\ \midrule
         Avg-Norm-Penalty  & $\|\hat{\theta}\|_2^2$ & \cmark      & \textbf{5.96} & \textbf{7.10} & 6.83\\ \midrule
         Avg-DiffNorm-Penalty & $\|\hat{\theta} - \hat{\theta}_1\|_2^2$ & \cmark  & 5.85 & 6.84 & \textbf{6.89} \\
        \bottomrule
\end{tabular}
\end{sc}
\end{small}
\end{center}
\vskip -0.1in
\end{table*}

\section{Problem Setup and Notations}\label{sec:note}

\textbf{Notation.} For any matrix $A$, we use $\| A \|_2$ to denote its $\ell_2$ operator norm and use $\mathrm{tr}\{ A \}$ to denote its trace. The $i$-th largest eigenvalue of $A$ is denoted as $\mu_i(A)$. The transposed matrix of $A$ is denoted as $A^T$. And the inverse matrix of $A$ is denoted as $A^{-1}$. The notation $a = o(b)$ and $a \ll b$ mean that $a/b \to 0$, $a = \omega(b)$ means that $a/b \to \infty$, $a = O(b)$ means that $a/b$ is bounded, and $a \asymp b$ means $a = O(b)$ as well as $b = O(a)$. 

Given the significant performance improvements achieved through ensembling, we seek to understand its benefits in this section. To this end, we analyze its effects within the framework of over-parameterized linear regression.

To be specific, the pre-training process is taken on Task 1, where $n$ i.i.d. training examples $(x_1, y_1), \dots, (x_n, y_n)$ from distribution $\cD$ take values in $\mathbb{R}^p \times \mathbb{R}$ and obey the following linear model with parameter $\theta \in \mathbb{R}^p$:
\begin{equation}
  \E[y_i \mid x_i] = x_i^T \theta.
  \label{eq:pre_model}
\end{equation}
To capture the strong performance of the pre-trained model, we consider the min-norm (``ridgeless'') estimator on Task 1:
\begin{equation}\label{eq:est_1}
    \hat{\theta}_1 := X^T (XX^T)^{-1} Y,
\end{equation}
where $X = [x_1, \dots, x_n]^T \in \mathbb{R}^{n \times p}$ and $Y = [y_1, \dots y_n]^T \in \mathbb{R}^n$. Starting with the pre-trained estimator $\hat{\theta}_1$, fine-tuning process is taken on Task 2, where $n$ i.i.d. training examples $(\tx_1, \ty_1), \dots, (\tx_n, \ty_n)$ are sampled from another distribution $\tilde{\cD}$ \footnote{For simplicity, we assume that the sample sizes for pre-training and fine-tuning are the same. However, our results remain valid even when using a larger dataset for pre-training compared to fine-tuning.}, as well as following another linear model:
\begin{equation}
  \E[\ty_i \mid \tx_i] = \tx_i^T \ttheta.
  \label{eq:fine_model}
\end{equation}
During the fine-tuning process, we consider two scenarios. The first is ``ridgeless” regression, defined by the following objective function:
\begin{equation*}
    \argmin_{\theta} \| \theta - \hat{\theta}_1 \|_2^2, \quad \text{s.t.} \quad \TX \theta = \TY,
\end{equation*}
where $\TX = [\tx_1, \dots, \tx_n]^T \in \mathbb{R}^{n \times p}$ and $\TY = [\ty_1, \dots, \ty_n]^T \in \mathbb{R}^n$. Accordingly, the estimator is as
\begin{equation}\label{eq:ets_2}
     \hat{\theta}_2 := \hat{\theta}_1 + \TX^T(\TX \TX^T)^{-1} (\TY - \TX \hat{\theta}_1).
\end{equation}
The second objective function incorporates a regularizer with parameter $\lambda$ to avoid overfitting, as in ridge regression:
\begin{equation*}
  \arg \min_{\theta} \frac{1}{n} \left\{ \| \TX \theta - \TY \|_2^2 + \lambda \| \theta - \hat{\theta}_1 \|_2^2 \right\}
\end{equation*}
which implies a solution as
\begin{equation}\label{eq:est_lam}
    \hat{\theta}_{\lambda} = \TX^T (\TX \TX^T + n \lambda I)^{-1} (\TY - \TX \hat{\theta}_1).
\end{equation}
With the pre-trained estimator and fine-tuning estimator, we also consider the ensemble ( weighted averaging) estimator:
\begin{equation}\label{eq:est_avg}
    \hat{\theta}_{\lambda}^\tau = (1 - \tau) \hat{\theta}_1 + \tau \hat{\theta}, \quad \text{where} \quad \hat{\theta} = \hat{\theta}_2,  \hat{\theta}_\lambda,
\end{equation}
with an averaging coefficient $0 \le \tau \le 1$. In our settings, the performance measures for an estimator $\hat{\theta}$ are the excess mean squared errors on Task 1:
\begin{align*}
    \mathcal{L}_{\mathrm{pre}} (\hat{\theta}) &:= \E_{(x_\star,y_\star,\theta)}\bigl[(x_\star^T \hat{\theta} - y_\star)^2\bigr] - \E_{(x_\star,y_\star)}\bigl[(x_\star^T \theta - y_\star)^2\bigr]\\
  &=
  \E_{x_\star}\bigl[(x_\star^T \hat{\theta} - x_\star^T \theta)^2\bigr]
  , 
\end{align*}
and Task 2:
\begin{align*}
\mathcal{L}_{\mathrm{ft}} (\hat{\theta}) &:= \E_{(\tx_\star,\ty_\star, \ttheta)}\bigl[(\tx_\star^T \hat{\theta} - \ty_\star)^2\bigr] - \E_{(\tx_\star,\ty_\star)}\bigl[(\tx_\star^T \tilde{\theta} - \ty_\star)^2\bigr]\\
  &=
  \E_{\tx_\star}\bigl[(\tx_\star^T \hat{\theta} - \tx_\star^T \ttheta)^2\bigr]
  ,
\end{align*}
where the variables $(x_\star, y_\star)$ and $(\tx_\star,\ty_\star)$ are independent copies of $(x_1, y_1)$ and $(\tx_1,\ty_1)$ respectively.

\paragraph{Motivation for the objective function.} The two objective functions used in the fine-tuning process serve to characterize the presence or absence of early stopping. Specifically, if early stopping is not employed, the overadaptation on downstream tasks can be described as ``ridgeless” regression. Conversely, if early stopping is applied, it can be viewed as utilizing a ridge regularizer \citep{lin2017optimal, lu2022sobolev}.

In further analysis, we adopt the following assumptions on settings above:
\begin{enumerate}
\item On Task 1, $x_i = \varSigma^{1/2} \eta_i$, where $\varSigma := \E[x_ix_i^T] = \text{diag}[\lambda_1, \dots, \lambda_p]$, and the components of $\eta_i$ are independent $\sigma_x$-subgaussian random variables with mean zero and unit variance;
\item On Task 2, $\tx_i = \tvar^{1/2} \tilde{\eta}_i$, where $\tvar := \E[\tx_i \tx_i^T] = \text{diag}[\tilde{\lambda}_1, \dots, \tilde{\lambda}_p]$, and the components of $\tilde{\eta}_i$ are also independent $\sigma_x$-subgaussian random variables with mean zero and unit variance;
\item $\mathbb{E}[y_i - x_i^T \theta | x_i, \theta]^2 = \mathbb{E} [\epsilon_i]^2 = \sigma^2 > 0$, $\mathbb{E}[\ty_i - \tx_i^T \ttheta | \tx_i, \ttheta]^2 = \mathbb{E} [\tep_i]^2 = \tsig^2 > 0$;
\item The true model parameters $\theta, \ttheta$ are independent of samples, and could be decomposed as 
\begin{equation*}
    \theta = \theta_c + \alpha_1, \quad \ttheta = \theta_c + \alpha_2,
\end{equation*}
where $\theta_c, \alpha_1, \alpha_2$ are independent with each other, as well as holding $\| \theta_c \|_2^2 < \infty$, $\mathbb{E} [\alpha_1 \alpha_1^T] = \zeta_1 I_p$ and $\mathbb{E} [\alpha_2 \alpha_2^T] = \zeta_2 I_p$, where $\zeta_1, \zeta_2 > 0$. 
\end{enumerate}

\paragraph{Discussion on the setting.} Pre-trained models often achieve reasonably good generalization across various tasks, and fine-tuning serves to enhance their performance on some specific tasks. Thus, it is reasonable to assume that, despite the differences between Task 1 and Task 2, they share many similarities, which makes the ``benign overfitting'' estimator $\hat{\theta}_1$ a good ``initial point'' in fine-tuning process. This leads us to posit that the true model parameters $\theta$ and $\ttheta$ share significant information, represented by $\theta_c$, while also exhibiting some differences characterized by $\alpha_1$ and $\alpha_2$. Additionally, we assume that $\varSigma$ and $\tvar$ share a large amount of same eigenvectors (see Condition~\ref{cond:eigen}), further reflecting the similarity between Task 1 and Task 2.

\paragraph{The connection between theoretical and empirical results.}
The theory aims to provide explanations for the benefits of ensemble, adopting a linear setting for intuitive insights. Such simplifications have been widely used in prior works \citep{mallinar2022benign, kumar2022fine}.
The connection between theoretical and empirical results can be established by adopting an NTK explanation, as fine-tuning results in parameters close to pretraining points. Specifically, for a nonlinear neural network $f(x, \vartheta)$ in the NTK regime, we approximate it using a first-order Taylor expansion $f(x, \vartheta) \approx f(x, \vartheta_0) + \nabla_\vartheta f(x, \vartheta_0)^T (\vartheta - \vartheta_0)$. Comparing this with the linear setting $y = x^T \theta_*$, we can interpret the ``features'' in neural networks, i.e., $\nabla_\vartheta f(x, \vartheta_0)$, as the input x in the linear model, and the trainable parameter $\vartheta - \vartheta_0$ as the parameter $\theta_*$ in $y = x^T \theta_*$. Since $f(x, \vartheta_0)$ remains unchanged during training, its effect can be disregarded in this simplification. And in our linear setup, the high-dimensional assumption on $x$ (see Condition~\ref{cond:order}) can characterize the ``features'' $\nabla_\vartheta f(x, \vartheta_0)$ in overparameterized neural networks.

\section{Main Theorems and Interpretations}\label{sec:re}
In this section, we present the test performance of various estimators and provide explanations for the improvement in both generalization on fine-tuning tasks and forgetting mitigation on pre-training tasks achieved through model ensemble, highlighting the ``bias-variance” trade-off phenomenon.
To simplify our explanations, we consider the covariance matrices for $x, \tx$ as follows, 
\begin{condition}\label{cond:eigen}
Denoting $\varSigma := \mathrm{diag} \{ \lambda_1, \dots, \lambda_p \}$ and $\tvar := \mathrm{diag} \{ \tilde{\lambda}_1, \dots, \tilde{\lambda}_p \}$, we have
\begin{equation*}
\lambda_i = \left\{ 
\begin{aligned}
& 1, \quad i = 1, \dots, k^*,\\
& \gamma, \quad i > k^*,
\end{aligned}
\right. \quad \tilde{\lambda}_i = \left\{ 
\begin{aligned}
& 1, \quad i = 1, \dots, k^*,\\
& \tgam, \quad k^* < i \le \tilde{p},\\
& 0, \quad i > \tilde{p}.
\end{aligned}
\right.
\end{equation*}    
\end{condition}
And our main theorem is also based on the following condition:
\begin{condition}\label{cond:order}
We consider the following three items:
 \begin{enumerate}
\item (good performance of pre-trained model) For Task 1, there exists a constant $0 < \xi < 1$, such that
\begin{equation*}
  k^* = O(1), \quad  p = \omega(n), \quad p = o(n^{1+ \xi}),
\end{equation*}
\item (sparsity of Task 2) For Task 2, the following conditions hold
\begin{equation*}
    \tilde{p} > n, \quad \tilde{p} \asymp n, \quad \tilde{p} \tgam \asymp 1, \quad \tilde{p} \tgam > 2 c_1 \tsig^2 / \zeta_2,
\end{equation*}
where $c_1 > 0$ is a constant only depending on $ \sigma_x$.
\item (non-negligible data noise and task difference) For the noise level and the ``bias'' parameter $\alpha_1, \alpha_2$, we have
\begin{align*}
  & \zeta_1 = O(n^{-\xi}), \quad, \zeta_2 \asymp \sigma^2 \asymp \tilde{\sigma}^2, \quad \zeta_2 = o(1),\\
  & \zeta_2 = \omega(\max\{n^{- (1 - \xi)/2}, n^{- \xi} \}). 
\end{align*}

 \end{enumerate}
\end{condition}

\paragraph{Discussion on the conditions.} 
Condition~\ref{cond:eigen} describes a high-dimensional eigenvalue structure characterized by several ``large" eigenvalues alongside many ``small" ones. This structure aligns with the eigenvalue decay observed in general kernel matrices within deep neural networks \citep{fan2020spectra, li2024eigenvalue}. We can further validate this assumption by analyzing the eigenvalue distribution of the Hessian matrix in practical models using PyHessian \citep{yao2020pyhessian} and variants of Lanczos algorithms\citep{zhang2024transformers}, which confirms the presence of several ``large'' eigenvalues and many ’’small'' eigenvalues. 
In Condition~\ref{cond:order}, the first item reflects the ``benign" overfitting scenario discussed in \citet{bartlett2020benign}, which helps to characterize the good performance of the pre-trained model. The second item delineates the ``sparse" structure of Task 2, i.e, there are many zero eigenvalues in $\tvar$. 
Such sparse structure observed in fine-tuning tasks reflects the nature of knowledge specialization across different inputs. While pretraining involves diverse inputs encompassing broad knowledge, fine-tuning is performed on specific tasks with a narrower scope, leading to a ``sparse'' structure in our theoretical formulation.
The third item outlines certain conditions that address the significance of data noise and the differences between the two tasks. Since the specific order relationships are technical assumptions intended to clarify our theoretical results more clearly, we believe that relaxing them will not compromise our theoretical intuition. We consider this a possible direction for further exploration.

Before delving into our main results in Theorem~\ref{thm}, we introduce two notations:
\begin{small}
  \begin{align*}
    \lambda' &:= \frac{\tsig^2}{n \zeta_2},\\
    \tau'(\lambda) &:= \zeta_2 \mathrm{tr}\{ (\TX \TX^T + n \lambda I)^{-1} \TX \tvar \TX^T\} \\
   & \left(\tsig^2 \mathrm{tr}\{ (\TX \TX^T + n \lambda I)^{-2} \TX \tvar \TX^T\}    \right. \\
   & \left. +\zeta_2 \mathrm{tr} \{ (\TX \TX^T + n \lambda I)^{-1} \TX \TX^T (\TX \TX^T + n \lambda I)^{-1} \TX \tvar \TX^T \}   \right)^{-1},
\end{align*}     
\end{small}
The main theorem is then stated as follows:
\begin{theorem}\label{thm}
For any $\sigma_x, \xi$ defined above, as Condition~\ref{cond:eigen}  and Condition~\ref{cond:order} are satisfied, there exists a constant $c>1$ such that for $\delta \in (0,1)$ and $\ln(1 / \delta) < n^{\xi} / c$,
with probability at least $1 - \delta$ over $X, \TX$, 
\begin{enumerate}
\item (the effectiveness of regularization.) for any $0 < \lambda \le 2 \lambda'$, we have
\begin{equation*}
   \mathcal{L}_{\mathrm{ft}}(\hat{\theta}_{\lambda}) < \mathcal{L}_{\mathrm{ft}}(\hat{\theta}_2) < \mathcal{L}_{\mathrm{ft}}(\hat{\theta}_1).
\end{equation*}
\item (forgetting mitigation with model ensemble.) for any $0 < \lambda \le 2 \lambda'$ and $\tau'(\lambda)/2 \le \tau < 1$, we have
\begin{footnotesize}
 \begin{equation*}
    \mathcal{L}_{\mathrm{pre}}(\hat{\theta}^{\tau}_\lambda) + \mathcal{L}_{\mathrm{ft}}(\hat{\theta}^{\tau}_\lambda) <  \mathcal{L}_{\mathrm{pre}}(\hat{\theta}_{\lambda}) + \mathcal{L}_{\mathrm{ft}}(\hat{\theta}_{\lambda}) < \mathcal{L}_{\mathrm{pre}}(\hat{\theta}_2) + \mathcal{L}_{\mathrm{ft}}(\hat{\theta}_2).
\end{equation*}      
\end{footnotesize}
\item (improving performance on ensemble.) for any $0 \le \lambda < \lambda'$ and $\tau'(\lambda) \le \tau < 1$, we have
\begin{equation*}
     \mathcal{L}_{\mathrm{ft}}(\hat{\theta}^{\tau}_\lambda) < \mathcal{L}_{\mathrm{ft}}(\hat{\theta}_{\lambda}) .
\end{equation*}
\end{enumerate}
\end{theorem}
The detailed proof is in Appendix~\ref{pf:thm}. The results highlight three key insights: (i) selecting an appropriate regularizer during fine-tuning helps reduce overadaptation on noisy samples, leading to improved generalization on fine-tuning task; ensembling the pre-trained and fine-tuned models can decrease overadaptation further, then (ii) enhances performance on fine-tuning task, as well as (iii) mitigating forgetting phenomenon on pre-training task. These benefits can be understood through a \emph{ ``bias-variance'' trade-off phenomenon}:
\begin{enumerate}
\item Both $\mathcal{L}_{\mathrm{pre}}$ and $\mathcal{L}_{\mathrm{ft}}$ contain ``bias'' term and ``variance'' term. 
\item The pre-trained estimator $\hat{\theta}_1$ is mainly dominated by ``bias'' terms, as it is induced from a sufficiently high-dimensional distribution (see Condition~\ref{cond:order}). 
It performs poorly on Task 2 and achieves good performance on Task 1, because it only contains the information in pre-training process, and lacks information specific to Task 2, resulting in a small ``bias'' term in $\mathcal{L}_{\mathrm{pre}}$, as well as a large ``bias" term in $\mathcal{L}_{\mathrm{ft}}$. 
\item On the other hand, the ``ridgeless" estimator $\hat{\theta}_2$, though it minimizes ``bias" error, overfits the noisy training data during fine-tuning, causing a significant ``variance" term in $\mathcal{L}_{\mathrm{ft}}$ and both large ``bias'' term and large ``variance'' term in $\mathcal{L}_{\mathrm{pre}}$.
\item Introducing a proper regularizer has the ability to achieve better performance on $\mathcal{L}_{\mathrm{ft}}$ and $\mathcal{L}_{\mathrm{pre}} + \mathcal{L}_{\mathrm{ft}}$, by balancing the ``bias" and ``variance" errors more effectively. 
\item The improved generalization on both $\mathcal{L}_{\mathrm{ft}}$ and $\mathcal{L}_{\mathrm{pre}} + \mathcal{L}_{\mathrm{ft}}$ from model ensemble results from further balancing these error terms with a properly chosen weight $\tau$, which applies to both the ``ridgeless" estimator $\hat{\theta}_2$ and the ridge-regularized estimator $\hat{\theta}_\lambda$.
\end{enumerate}

\paragraph{Comparing with previous viewpoints.} From a traditional statistical perspective, which mainly focuses on limited model complexity, increasing the model complexity typically results in a higher ``variance'' error and a lower ``bias'' error \citep{zhang_2023_ltbook}. This trade-off suggests that overfitting noisy training data leads to poor generalization and high test error. However, recent advancements have introduced the concept of ``benign overfitting'' \citep{bartlett2020benign}, which suggests that sufficiently large models can achieve superior performance despite overfitting. In our analysis, the pre-training process mainly operates on a high-dimensional distribution $\mathcal{D}$, facilitating strong performance on Task 1 and aligning with the principles of ``benign overfitting''. Conversely, the fine-tuning phase focuses on a ``sparse'' structure, i.e, $\tilde{\mathcal{D}}$, associated with limited model complexity. This limited complexity explains the observed harmful overfitting during fine-tuning.

\paragraph{Empirical validation.} We also conduct simulations across diverse settings to validate Theorem~\ref{thm}. The details and results, summarized in Appendix~\ref{simu}, demonstrate the strong performance of the ensemble model on both pre-training and fine-tuning tasks, aligning well with our theoretical findings.

\section{Proof Sketches of Theorem~\ref{thm}}
In this section, we summarize the proof sketch of Theorem~\ref{thm}, which mainly contains two steps, and the detailed proof is in Appendix~\ref{pf:thm}. For simplification, we take the following two notations in further analysis:
\begin{equation*}
  P_{\TX, \lambda} := \TX^T (\TX \TX^T + n \lambda I)^{-1} \TX, \quad P_{\TX} := \TX^T (\TX \TX^T)^{-1} \TX.
\end{equation*}

\subsection{Excess Risks Approximation}
First, we show that with a high probability, the excess risks corresponding to estimators in \eqref{eq:est_1}, \eqref{eq:ets_2}, \eqref{eq:est_lam} and \eqref{eq:est_avg} could be expressed as:
\begin{lemma}\label{lem:approx}
As Condition~\ref{cond:eigen} and Condition~\ref{cond:order} are satisfied, there exist a constant $c > 1$ such that for $\delta \in (0, 1)$ and $\ln (1 / \delta) < n^\xi / c$, with probability at least $1 - \delta$ over $X, \tilde{X}$, if $0 \le \lambda \le \tsig^2 / (n \zeta_2)$, we have
\begin{footnotesize}
  \begin{align*}
\mathcal{L}_{\mathrm{ft}}(\hat{\theta}_1) &\approx \zeta_2 \mathrm{tr}\{ \tvar \},\\
\mathcal{L}_{\mathrm{ft}}(\hat{\theta}_2) & \approx \zeta_2 \mathrm{tr}\{ (I -P_{\TX})^2 \tvar \} + \tsig^2 \mathrm{tr}\{ (\TX \TX^T)^{-2} \TX \tvar \TX^T\}, \\
\mathcal{L}_{\mathrm{ft}}(\hat{\theta}_\lambda) & \approx \zeta_2 \mathrm{tr}\{ (I - P_{\TX, \lambda})^2 \tvar \} \\
& \quad+ \tsig^2 \mathrm{tr}\{ (\TX \TX^T + n \lambda I)^{-2} \TX \tvar \TX^T\}, \\
\mathcal{L}_{\mathrm{ft}}(\hat{\theta}^{\tau}_{\lambda}) & \approx  \zeta_2 \mathrm{tr}\{ (I - \tau P_{\TX, \lambda})^2 \tvar \}\\
& \quad+ \tau^2 \tsig^2 \mathrm{tr}\{ (\TX \TX^T + n \lambda I)^{-2} \TX \tvar \TX^T\},
\end{align*}      
\end{footnotesize}
and
\begin{footnotesize}
  \begin{align*}
\mathcal{L}_{\mathrm{pre}}(\hat{\theta}_1) & \ll \mathcal{L}_{\mathrm{pre}}(\hat{\theta}_2), \quad \mathcal{L}_{\mathrm{pre}}(\hat{\theta}_1) \ll \mathcal{L}_{\mathrm{pre}}(\hat{\theta}_\lambda),  \\
 \mathcal{L}_{\mathrm{pre}}(\hat{\theta}_1)  &\ll \mathcal{L}_{\mathrm{pre}}(\hat{\theta}_\lambda^\tau),\\
 \mathcal{L}_{\mathrm{pre}}(\hat{\theta}_2) &\approx \zeta_2 \mathrm{tr}\{ (\TX \TX^T)^{-1} \TX \varSigma \TX^T \} \\
 & \quad + \tsig^2 \mathrm{tr}\{ (\TX \TX^T)^{-2} \TX \varSigma \TX^T\},\\
 \mathcal{L}_{\mathrm{pre}}(\hat{\theta}_\lambda) & \approx \zeta_2 \mathrm{tr}\{ P_{\TX, \lambda}^2 \varSigma \} + \tsig^2 \mathrm{tr}\{ (\TX \TX^T + n \lambda I)^{-2} \TX \varSigma \TX^T\},\\
 \mathcal{L}_{\mathrm{pre}}(\hat{\theta}_\lambda^\tau) & \approx \zeta_2 \tau^2 \mathrm{tr}\{ P_{\TX, \lambda}^2 \varSigma \} \\
 & \quad + \tau^2 \tsig^2 \mathrm{tr}\{ (\TX \TX^T + n \lambda I)^{-2} \TX \varSigma \TX^T\}.
\end{align*}         
\end{footnotesize}
\end{lemma}

 \begin{proof}
 See details in Appendix~\ref{pf:decom} and \ref{pf:bound}.    
 \end{proof}

Using Lemma~\ref{lem:approx}, the terms $\mathcal{L}_{\mathrm{ft}}$ and $\mathcal{L}_{\mathrm{pre}} + \mathcal{L}_{\mathrm{ft}}$ in Theorem~\ref{thm} can be primarily determined by two key factors: the terms related to $\zeta_2$ (the difference between the pre-training and fine-tuning tasks), which can be denoted as the ``bias'' terms, and the terms related to $\tsig$ (the variance of data noise in the fine-tuning task), which can be denoted as the ``variance'' term.
Insufficient fine-tuning leads to a large ``bias'' and small ``variance'', while overadaptation in fine-tuning results in large ``variance'' and small ``bias''. By effectively balancing this trade-off, model ensembling can achieve superior performance.

\subsection{Estimator Performances Comparison}
\paragraph{The effectiveness of regular.} After obtaining the results in Lemma~\ref{lem:approx}, we could compare the excess risk on different estimators. The results
\begin{equation}\label{eqf:com1}
    \mathcal{L}_{\mathrm{ft}}(\hat{\theta}_2) < \mathcal{L}_{\mathrm{ft}}(\hat{\theta}_1)
\end{equation}
could be induced directly from Condition~\ref{cond:order}. For the excess risk of ridge estimator, taking derivative with respect to $\lambda$ on its approximation
\begin{equation*}
    \zeta_2 \mathrm{tr}\{ (I - P_{\TX, \lambda})^2 \tvar \} + \tsig^2 \mathrm{tr}\{ (\TX \TX^T + n \lambda I)^{-2} \TX \tvar \TX^T\} ,
\end{equation*}
we find that such excess risk will decrease with the increase of $\lambda$ within the range $0 \le \lambda \le \tsig^2 / (n \zeta_2)$,  which implies 
\begin{equation}\label{eqf:com2}
    \mathcal{L}_{\mathrm{ft}}(\hat{\theta}_\lambda) < \mathcal{L}_{\mathrm{ft}}(\hat{\theta}_2), \quad \forall 0 < \lambda \le 2 \tsig^2 / (n \zeta_2).
\end{equation}

\paragraph{Forgetting mitigation with model ensemble.} The analysis is similar. With the results in Lemma~\ref{lem:approx}, $\mathcal{L}_{\mathrm{pre}}(\hat{\theta}_\lambda) + \mathcal{L}_{\mathrm{ft}}(\hat{\theta}_\lambda)$ is mainly dominated by
\begin{align*}
 &\quad \zeta_2 \mathrm{tr}\{ (I - P_{\TX, \lambda})^2 \tvar \} + \zeta_2 \mathrm{tr} \{ P_{\TX, \lambda}^2 \tvar \}\\
 & + 2 \tsig^2 \mathrm{tr}\{ (\TX \TX^T + n \lambda I)^{-2} \TX \tvar \TX^T\} .
\end{align*}
Taking derivative with respect to $\lambda$, we find that such term will decrease while $0 \le \lambda \le 2 \tsig^2 / (n \zeta_2)$, which implies that
 \begin{equation}\label{eq:comp1}
 \begin{aligned}
     \mathcal{L}_{\mathrm{pre}}(\hat{\theta}_\lambda) + \mathcal{L}_{\mathrm{ft}}(\hat{\theta}_\lambda)  < \mathcal{L}_{\mathrm{pre}}(\hat{\theta}_2) &+ \mathcal{L}_{\mathrm{ft}}(\hat{\theta}_2),\\
     & \forall 0 < \lambda \le 2 \tsig^2 / (n \zeta_2).
 \end{aligned}
\end{equation}     
And considering the benefits of model ensemble, for any fixed $\lambda \in (0, 2 \tsig^2 / (n \zeta_2) ]$, $\mathcal{L}_{\mathrm{pre}}(\hat{\theta}_\lambda^\tau) + \mathcal{L}_{\mathrm{ft}}(\hat{\theta}_\lambda^\tau)$ is mainly dominated by 
  \begin{align*}
  & \quad  \zeta_2 \mathrm{tr}\{ (I - \tau P_{\TX, \lambda})^2 \tvar \}\\
  &+ \tau^2 \zeta_2 \mathrm{tr} \{ (\TX \TX^T + n \lambda I)^{-2} \TX \TX^T \TX \tvar \TX^T \}  \\
  &+ 2 \tau^2 \tsig^2 \mathrm{tr}\{ (\TX \TX^T + n \lambda I)^{-2} \TX \tvar \TX^T\}  . 
\end{align*}  
While we take derivative with respect to $\tau$, such term will increase with the increasing of $\tau$ as $\tau'(\lambda) / 2 \le \tau \le 1$ . So we have
 \begin{equation}\label{eq:comp2}
 \begin{aligned}
     \mathcal{L}_{\mathrm{pre}}(\hat{\theta}^{\tau}_\lambda) + \mathcal{L}_{\mathrm{ft}}(\hat{\theta}^{\tau}_\lambda) <  \mathcal{L}_{\mathrm{pre}}(\hat{\theta}_{\lambda}) &+ \mathcal{L}_{\mathrm{ft}}(\hat{\theta}_{\lambda}), \\
     & \forall \tau'(\lambda) / 2 \le \tau < 1.
 \end{aligned} 
\end{equation}

\paragraph{Improved fine-tuning performance on ensemble.}
Finally, for any fixed $\lambda$ with range $[0, \tsig^2 / (n \zeta_2))$, we could take derivative with respect to $\tau$ on the approximated excess risk of ensemble estimator:
\begin{equation*}
     \zeta_2 \mathrm{tr}\{ (I - \tau P_{\TX, \lambda})^2 \tvar \}  + \tau^2 \tsig^2 \mathrm{tr}\{ (\TX \TX^T + n \lambda I)^{-2} \TX \tvar \TX^T\},
\end{equation*}
we found such excess risk will increase with the increasing of $\tau$ while $\tau'(\lambda) \le \tau \le 1$, so we have
\begin{equation}\label{eqf:com3}
    \mathcal{L}_{\mathrm{ft}}(\hat{\theta}_\lambda^\tau) < \mathcal{L}_{\mathrm{ft}}(\hat{\theta}_\lambda), 
\end{equation}
while $0 \le \lambda < \tsig^2 / (n \zeta_2)$ and $\tau'(\lambda) \le \tau \le 1$.

Combing all of the results in \eqref{eqf:com1}, \eqref{eqf:com2}, \eqref{eq:comp1}, \eqref{eq:comp2} and \eqref{eqf:com3}, we could finish the proof of Theorem~\ref{thm}.

\section{Conclusion and Discussion}

In this work, we bridge the gap in understanding how ensembling pre-trained and fine-tuned models controls overadaptation, as well as enhancing both downstream performance and mitigating forgetting on upstream tasks. Motivated by surprising empirical findings showing that ensembling not only improves fine-tuning outcomes but also preserves pre-trained knowledge, we provide a formal theoretical analysis within an over-parameterized linear setting.
Our results reveal that ensembling mitigates overadaptation by effectively balancing the trade-off between ``bias'' and ``variance'' errors in excess risk—an issue that regularization alone may not fully resolve. This theoretical insight is further supported by experiments and simulations, which closely align with our predictions.

Our results not only offer a deeper theoretical understanding of ensembling in the context of pre-trained models but also provide practical guidance for enhancing the performance of fine-tuning strategies. This work lays a foundation for future research into refining ensembling methods and exploring their application to broader machine learning tasks.

\section*{Acknowledgements}

This work was partially supported by NSF grant No. 2416897 and ONR grant No. N000142512318.

\section*{Impact Statement}

This paper contributes foundational research in the areas of model ensemble within the machine learning community. Our primary goal is to advance the theoretical understanding for the efficient performance of ensemble methods. Given the scope of this research, we do not anticipate immediate ethical concerns or direct societal consequences.
Therefore, we believe there are no specific ethical considerations or immediate societal impacts to be emphasized in the context of this work.

\nocite{langley00}

\bibliography{myrefs}
\bibliographystyle{icml2025}

\newpage
\appendix
\onecolumn

\section{Experimental Details}~\label{appendix:exp_details}

\subsection{Hyperparameter Searching}

We have conducted hyper-parameter searching on the fine-tuning process and the ensemble process. We fine-tune the models with a global batch size of 16, and an epoch of 1 using Adam optimizer on 8 GPUs. 
To select a suitable learning rate and penalty, we search the learning rate on $\{5 \times 10^{-6},2 \times 10^{-6},10^{-6}\}$, and penalty coefficient $\lambda$ on $\{10^{-2},5 \times 10^{-3},2 \times 10^{-3},10^{-3}\}$.
We also search the ensemble weight $\tau$ uniformly on $\{0.1,0.2,0.3,0.4,0.5,0.6,0.7,0.8,0.9\}$.

To preliminarily validate the performance and choose the hyper-parameter, we have a carefully curated instruction-following dataset. The validation dataset consists of multi-turn conversations between the user and the assistant, covering writing, reasoning, coding, math, STEM, and humanities topics. We prompt GPT4 using the prompt \textit{``Help me generate 3 sets of 2-turn instructions to evaluate the \{category\} ability of LLMs. The instructions for the second turn need to be highly relevant to the first turn. The following is an example.\textbackslash n\textbackslash n\textbackslash n EXAMPLE:\{example\}\textbackslash n TURN1:\{turn1\}\textbackslash n TURN2:\{turn2\}\textbackslash n''}, where \textit{\{category\}} corresponds to one of the 8 categories in MT-Bench and \textit{\{example\}} is one example from MT-Bench. In this way, we obtain a validation dataset that is highly similar to MT-Bench. Specifically, our validation dataset contains 600 samples, evenly distributed across the 8 categories in MT-Bench. We then represent the performance using the loss calculated on the validation dataset. 

\subsection{Implementation} We implemented our fine-tuning code based on Huggingface Transformers\footnote{https://github.com/huggingface/transformers} and Accelerate\footnote{https://github.com/huggingface/accelerate} libraries, where Fully Sharded Data Parallel~\cite{zhao2023pytorch} is utilized for model parallel training and acceleration. Our training and evaluation are conducted on 8 NVIDIA H100 GPUs.

\subsection{Extension to LoRA}

\begin{table*}[h!]
\centering
\begin{tabular}{lccc}
\toprule
\textbf{Methods} & \textbf{MT-Bench} & \textbf{Commonsense-QA} & \textbf{MMLU} \\
\midrule
DiffNorm-Penalty + Ensemble & 5.85 & \textbf{74.49} & 63.97 \\
LoRA                        & 5.83 & 73.71          & 65.29 \\
LoRA + Ensemble             & \textbf{5.99} & 73.87          & \textbf{65.31} \\
\bottomrule
\end{tabular}
\caption{Performance comparison of different LoRA-based methods.}
\label{table:lora}
\end{table*}

We also conduct experiments with LoRA \citep{hu2021loralowrankadaptationlarge}. Specifcally, we set $r=32$, $\alpha=32$, $dropout=0.01$, and target modules to \textit{q-projection, v-projection}. The results are shown in Table~\ref{table:lora}. We first observe that LoRA can mitigate overadaptation as well but tends to forget more in certain benchmarks, such as Commonsense-QA in comparison with DiffNorm-Penalty. On top of that, it is observed that further ensembling with the pre-trained model yields additional performance improvement in all benchmarks. Such results highlight the benefits of ensemble methods.

\subsection{Variance of MT-Bench}

We also examine the variance of the MT-Bench by fine-tuning the model 5 trials with different seeds, under Norm-Penalty and ensembling with $\tau=0.8$. The results are shown in Table~\ref{table:variance}. Overall, we observe a standard variance of 0.06, which is sufficiently small in comparison to the score gaps in Table~\ref{MT-Bench}. 

\begin{table*}[h!]
\centering
\begin{tabular}{cccccc|c}
\toprule
\textbf{Trial 1} & \textbf{Trial 2} & \textbf{Trial 3} & \textbf{Trial 4} & \textbf{Trial 5} & & \textbf{Standard Deviation} \\
\midrule
5.84 & 5.91 & 5.95 & 5.96 & 6.01 & & \textbf{0.06} \\
\bottomrule
\end{tabular}
\caption{Variance estimation of ensembled Norm-Penalty with $\tau = 0.8$ on LLaMA-3-8b.}
\label{table:variance}
\end{table*}

\section{Empirical Validation for Theorem~\ref{thm}}\label{simu}

To validate Theorem~\ref{thm}, we first utilize artificial datasets, where we construct pre-trained and fine-tuned datasets based on 4 diverse groups of parameters, respectively. Specifically, consider 4 cases for different eigen-value parameter $\gamma$ and the size of pre-trained set $n$: (a) $\gamma=n^{-1.0},~n=40$; (b) $\gamma=n^{-1.5},~n=40$; (c) $\gamma=n^{-1.0},~n=40$; (d) $\gamma=n^{-1.5},~n=60$. For each case, we set data dimension as $p=10000$, the size of test data as $1000$. We generate a pre-train dataset and a fine-tune dateset from two normal distributions $\cN(0,\Sigma_1)$ and $\cN(0,\Sigma_2)$, respectively, where $\Sigma_1$ has eigenvalues $\lambda_1=1,\lambda_2=\ldots=\lambda_p=n^{-1.5}$, and $\Sigma_2$ has eigenvalues $\lambda_1=1,\lambda_2=\ldots=\lambda_n=n^{-1},\lambda_{n+1}=\ldots=\lambda_p=0$. The ground-truth parameter for the pre-train and fine-tune $\theta_c+\alpha_1$ and $\theta_c+\alpha_2$, where $\|\theta_c\|_2=1$, $\alpha_1\sim\cN(0,0.01^2I_p)$, and $\alpha_2\sim\cN(0,0.1^2I_p)$. The variance of data noise is $0.1^2$. After obtaining the pre-trained estimator $\hat{\theta}_1$, we fine-tune the estimator on the other dataset to compute the ``min-norm" estimator $\hat{\theta}_2$ according to \eqref{eq:ets_2}, and the estimator with regularizer $\hat{\theta}_\lambda$. We tune the hyper-parameter $\lambda$ within a range and choose the $\lambda$ that achieves the best excess risk, which is $\lambda=0.0001$. 
Finally, we calculate a group of ensemble estimators $\hat{\theta}_{\lambda}^{\tau}$ with $\tau$ ranging from $0$ to $1$, and plot the curve of the error on the pretrain task versus the error in the fine-tuning task for the group of $\hat{\theta}_{\lambda}^{\tau}$ in four cases in comparison with the fine-tuned estimator with difference $\lambda$, the ``min-norm" estimator and the $\hat{\theta}_\lambda$ in Figure \ref{fig:toy}. According to Figure \ref{fig:toy}, the performance curve for the ensemble estimator $\hat{\theta}_{\lambda}^{\tau}$ achieves better trade-off on the two tasks compared to fine-tuning estimators, which aligns with Theorem \ref{thm}.

Additionally, to validate the performance only on the fine-tuned task, we also consider the four settings mentioned above. To simulate the realistic situation that it is difficult to find the best $\lambda$, and we can only tune it into a small range, we take $\lambda=1e-7$. Finally, we calculate a group of ensemble estimators $\hat{\theta}_{\lambda}^{\tau}$ with $\tau$ ranging from $0$ to $1$, and plot the curve of excess risk for the group of $\hat{\theta}_{\lambda}^{\tau}$ in four cases in comparison with the pre-trained estimator, the ``min-norm" estimator and the $\hat{\theta}_\lambda$ in Figure \ref{fig:toy2}. The figure implies that if we tune the ensemble parameter $\tau$ to the optimal, the ensemble estimator $\hat{\theta}_{\lambda}^{\tau}$ performs best, and then the performance decreases in the order of the estimator with ridge regularization $\hat{\theta}_\lambda$, the ``min-norm" estimator $\hat{\theta}_2$ and the pre-trained estimator $\hat{\theta}_1$, which aligns with Theorem \ref{thm}.

\begin{figure}[h]  
    \centering
    \begin{subfigure}{0.4\textwidth}  
        \centering
        \includegraphics[width=\linewidth]{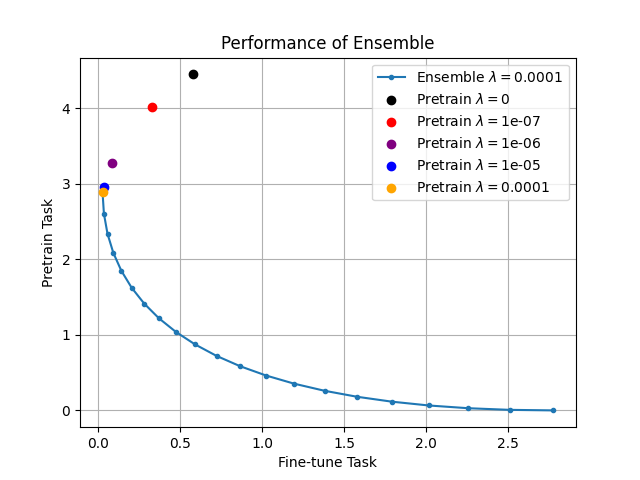}  
        \caption{$\gamma=0.025,n=40$}
        \label{fig:subfig1}
    \end{subfigure}
    \begin{subfigure}{0.4\textwidth}
        \centering
        \includegraphics[width=\linewidth]{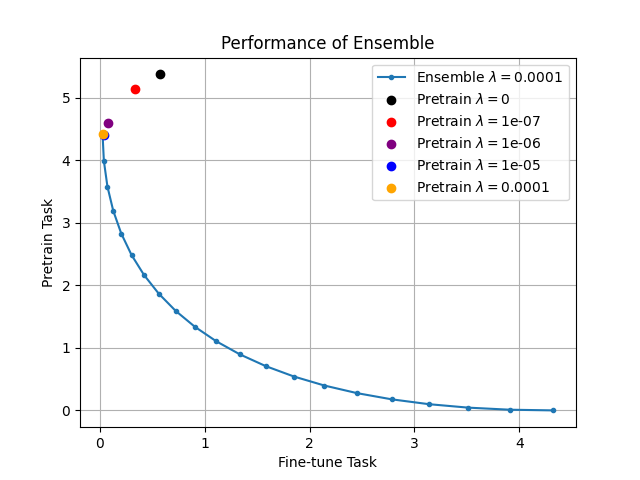}
        \caption{$\gamma=0.004,n=40$}
        \label{fig:subfig2}
    \end{subfigure}
    
    \begin{subfigure}{0.4\textwidth}
        \centering
        \includegraphics[width=\linewidth]{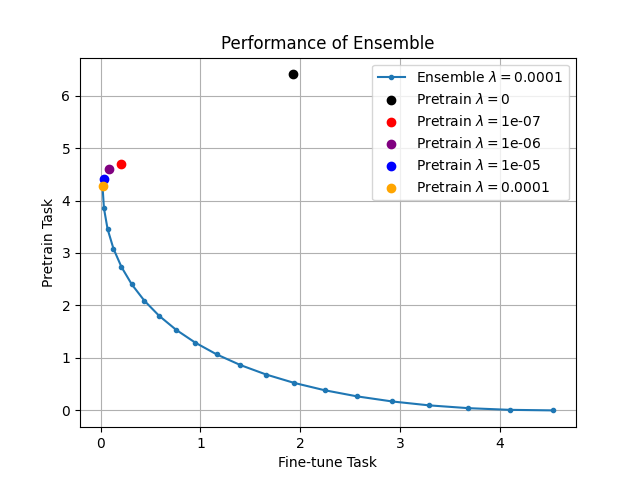}
        \caption{$\gamma=0.017,n=60$}
        \label{fig:subfig3}
    \end{subfigure}
    \begin{subfigure}{0.4\textwidth}
        \centering
        \includegraphics[width=\linewidth]{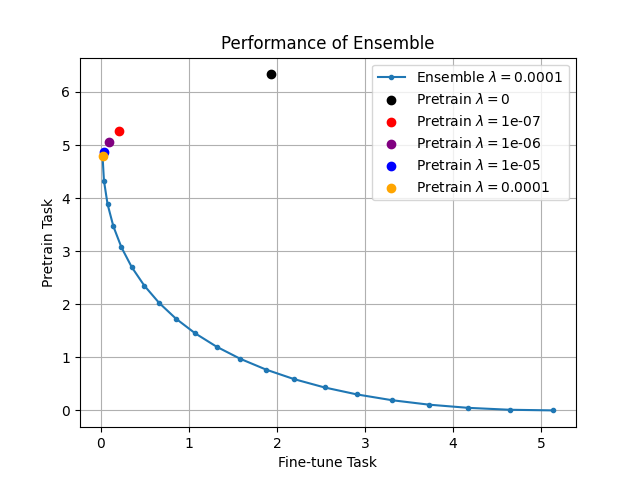}
        \caption{$\gamma=0.0022,n=60$}
        \label{fig:subfig4}
    \end{subfigure}

    \caption{Performance of Ensemble with dimension $p=10^{4}$}
    \label{fig:toy}
\end{figure}

\begin{figure}[h]  
    \centering
    \begin{subfigure}{0.4\textwidth}  
        \centering
        \includegraphics[width=\linewidth]{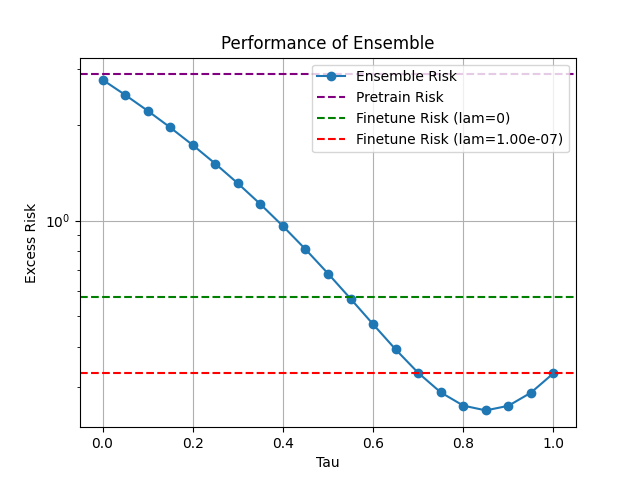}  
        \caption{$\gamma=0.025,n=40$}
        \label{fig:subfig21}
    \end{subfigure}
    \begin{subfigure}{0.4\textwidth}
        \centering
        \includegraphics[width=\linewidth]{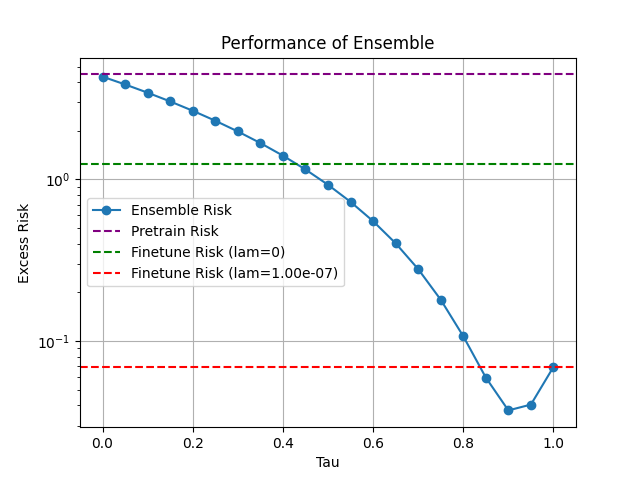}
        \caption{$\gamma=0.004,n=40$}
        \label{fig:subfig22}
    \end{subfigure}
    
    \begin{subfigure}{0.4\textwidth}
        \centering
        \includegraphics[width=\linewidth]{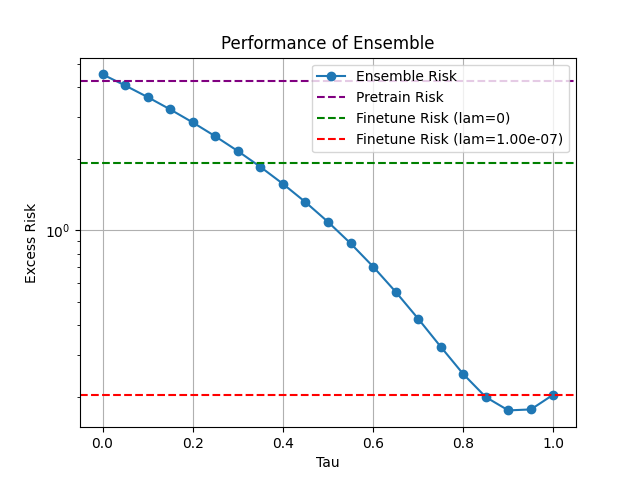}
        \caption{$\gamma=0.017,n=60$}
        \label{fig:subfig23}
    \end{subfigure}
    \begin{subfigure}{0.4\textwidth}
        \centering
        \includegraphics[width=\linewidth]{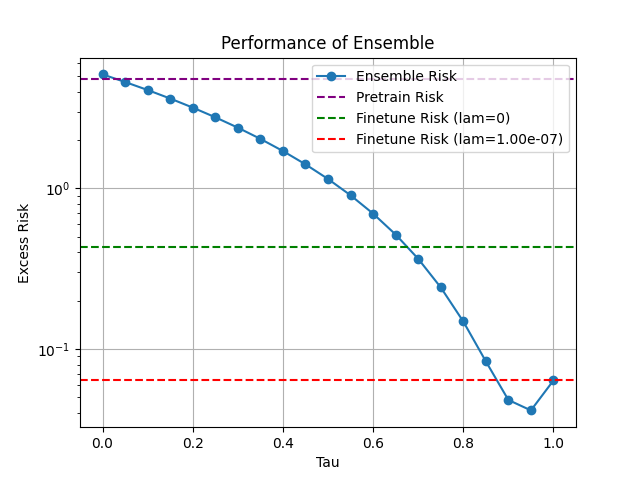}
        \caption{$\gamma=0.0022,n=60$}
        \label{fig:subfig24}
    \end{subfigure}

    \caption{Performance of Ensemble with dimension $p=10^{4}$}
    \label{fig:toy2}
\end{figure}

\clearpage
\newpage

\section{Proofs}\label{pf:thm}

\subsection{Notation and Constant List}
Before the main proof process, we denote several corresponding constants and notations in Table \ref{tab:const}:
\begin{table}[h]
    \centering
    \footnotesize
    \begin{tabular}{c|c}
    \toprule
         Symbol & Value / Expression\\
         \midrule
          $c'$ & $\max\{ 2, (1 + 16 \ln 3 \cdot \sigma_x^2 \cdot 54e) 32 \ln 3 \cdot \sigma_x^2 \cdot 54e\}$ \\
          \midrule
          $c$ & $> 256 \cdot (162e)^4 \sigma_x^4$ \\
          \midrule
          $c_1$ & $\max\{ 2 c' , (1 / c' - c')^{-1} \}$  \\  
         \midrule
         $c_2$ & $8(162e)^2 \sigma_x^2$ \\
         \midrule
         $c_3$ & $2$\\
         \midrule
         $\lambda_k$ & the $k$-th eigenvalue of matrix $\varSigma$, i.e, $\mu_k(\varSigma)$ \\
         \midrule
         $\tilde{\lambda}_k$ & the $k$-th eigenvalue of matrix $\tvar$, i.e, $\mu_k(\tvar)$ \\
         \midrule
         $r_k$ & $\frac{\sum_{j > k} \lambda_j}{\lambda_{k+1}}$\\
         \bottomrule
    \end{tabular}
    \caption{Constant and Notation List}
    \label{tab:const}
\end{table}

The definition of $r_k$ is the same as the definition in \citet{bartlett2020benign}. In \citet{bartlett2020benign}, the critical index $s^*(b)$ for a given $b>0$ is defined as
\begin{equation}\label{eq:defk}
  s^*(b) := \inf \{ k \geq 0 : r_k \geq bn \} .
\end{equation}
In our data settings, without lose of generality, we choose $b = 1$, and obtain the critical index $s^*(1) = k^*$ in $\varSigma$ as well as $\tvar$.

\subsection{Excess Risk Decomposition}\label{pf:decom}
The detailed analysis is start with a composition for excess risks. First, the estimators mentioned in main text could be expressed as:
\begin{align*}
 \hat{\theta}_1 &= X^T(XX^T)^{-1} (X \theta + \epsilon) = X^T (XX^T)^{-1} (X \theta_c + X \alpha_1 + \epsilon),\\
\hat{\theta}_2 &= \hat{\theta}_1 + \TX^T(\TX \TX^T)^{-1}(\TX \ttheta - \TX \hat{\theta}_1 + \tep)\\
&= [I - \TX^T (\TX \TX^T)^{-1} \TX] X^T (XX^T)^{-1} (X\theta_c + X \alpha_1 + \epsilon) + \TX^T (\TX \TX^T)^{-1} \TX (\theta_c + \alpha_2)  + \TX (\TX \TX^T)^{-1} \tep  \\
&= [X^T (XX^T)^{-1} X + \TX^T (\TX \TX^T)^{-1} \TX  - \TX^T (\TX \TX^T)^{-1} \TX X^T (XX^T)^{-1} X ] \theta_c \\
& \quad + [I - \TX^T (\TX \TX^T)^{-1} \TX] X^T (XX^T)^{-1} X \alpha_1\\
& \quad + \TX^T(\TX \TX^T)^{-1} \TX \alpha_2  + \TX (\TX \TX^T)^{-1} \tep + [I - \TX^T (\TX \TX^T)^{-1} \TX] X^T (XX^T)^{-1} \epsilon,\\
\hat{\theta}_\lambda &= \hat{\theta}_1 + \TX^T(\TX \TX^T + n \lambda I)^{-1} (\TX \tilde{\theta} - \TX \hat{\theta}_1 + \tep)\\
&= [I - \TX^T (\TX \TX^T + n \lambda I)^{-1}\TX] X^T (XX^T)^{-1} (X \theta_c + X \alpha_1 + \epsilon) \\
& \quad + \TX^T(\TX \TX^T + n \lambda I)^{-1} \TX (\theta_c + \alpha_2) + \TX^T (\TX\TX^T + n \lambda I)^{-1} \tep\\
&= [X^T(XX^T)^{-1} X + \TX^T(\TX \TX^T + n \lambda I)^{-1} \TX - \TX^T(\TX \TX^T + n \lambda I)^{-1} \TX X^T (XX^T)^{-1}X] \theta_c\\
& \quad  + [I - \TX^T(\TX \TX^T + n \lambda I)^{-1} \TX]X^T(XX^T)^{-1} X \alpha_1  + \TX^T(\TX \TX^T + n \lambda I)^{-1} \TX \alpha_2 \\
& \quad + [I - \TX^T(\TX \TX^T + n \lambda I)^{-1}\TX ] X^T (XX^T)^{-1} \epsilon + \TX^T (\TX \TX^T + n \lambda I)^{-1} \tep,\\
\hat{\theta}_{\lambda}^\tau &= (1 - \tau) \hat{\theta}_1 + \tau \hat{\theta}_\lambda \\
&= [X^T(XX^T)^{-1}X + \tau \TX^T(\TX \TX^T + n \lambda I)^{-1} \TX - \tau \TX^T(\TX \TX^T + n \lambda I)^{-1} \TX X^T(XX^T)^{-1}X] \theta_c \\
& \quad + [I - \tau \TX^T(\TX \TX^T + n \lambda I)^{-1}\TX] X^T (XX^T)^{-1} X \alpha_1 + \tau \TX^T(\TX\TX^T + n \lambda I)^{-1} \TX \alpha_2\\
& \quad + [I - \tau \TX^T(\TX \TX^T + n \lambda I)^{-1}\TX] X^T (XX^T)^{-1} \epsilon + \tau \TX^T (\TX\TX^T + n \lambda I)^{-1} \tep.
\end{align*}
Then focusing on the excess risks on Task 1 and Task 2, i.e, 
\begin{equation*}
 \mathcal{L}_{\mathrm{pre}}(\hat{\theta}) := \mathbb{E}_{x, \epsilon, \tep, \alpha_1, \alpha_2} (\hat{\theta} - \theta)^T \varSigma (\hat{\theta} - \theta), \quad
    \mathcal{L}_{\mathrm{ft}}(\hat{\theta}) := \mathbb{E}_{x, \epsilon, \tep, \alpha_1, \alpha_2} (\hat{\theta} - \ttheta)^T \tvar (\hat{\theta} - \ttheta),
\end{equation*}
we have the following results:
\begin{align*}
\mathcal{L}_{\mathrm{pre}}(\hat{\theta}_1) &= \theta_c^T [I - X^T(XX^T)^{-1}X] \varSigma [I - X^T(XX^T)^{-1} X] \theta_c + \zeta_1 \mathrm{tr}\{ [ I - X^T (XX^T)^{-1} X] \varSigma \} \\
& \quad + \sigma^2 \mathrm{tr}\{ (XX^T)^{-2} X \varSigma X^T\}\\
\mathcal{L}_{\mathrm{pre}}(\hat{\theta}_2\} &= \theta_c^T [I - X^T (XX^T)^{-1}X][I - \TX^T (\TX \TX^T)^{-1} \TX] \varSigma [I - \TX^T(\TX\TX^T)^{-1} \TX] [I - X^T(XX^T)^{-1} X] \theta_c \\
& \quad + \zeta_1 \mathrm{tr} \{ [I - X^T (XX^T)^{-1} X + X^T(XX^T)^{-1}X \TX^T(\TX \TX^T)^{-1}\TX ] \varSigma  \\
& \quad \quad \quad \quad \quad \quad [I - X^T (XX^T)^{-1} X +\TX^T(\TX \TX^T)^{-1}\TX X^T(XX^T)^{-1}X  ] \} \\
& \quad + \zeta_2 \mathrm{tr}\{ (\TX \TX^T)^{-1} \TX \varSigma \TX^T \} \\
& \quad + \sigma^2 \mathrm{tr}\{ (XX^T)^{-2}X [I - \TX^T(\TX\TX^T)^{-1}\TX] \varSigma [I - \TX^T(\TX\TX^T)^{-1}\TX] X^T\} + \tsig^2 \mathrm{tr}\{ (\TX \TX^T)^{-2} \TX \varSigma \TX^T\} \\
\mathcal{L}_{\mathrm{pre}}(\hat{\theta}_\lambda) &= \theta_c^T [I - X^T (XX^T)^{-1}X][I - \TX^T (\TX \TX^T + n \lambda I)^{-1} \TX] \varSigma [I - \TX^T(\TX\TX^T + n \lambda I)^{-1} \TX] [I - X^T(XX^T)^{-1} X] \theta_c \\
& \quad + \zeta_1 \mathrm{tr} \{ [I - X^T (XX^T)^{-1} X + X^T(XX^T)^{-1}X \TX^T(\TX \TX^T + n \lambda I)^{-1}\TX ] \varSigma \} \\
& \quad \quad \quad \quad \quad \quad [I - X^T (XX^T)^{-1} X +\TX^T(\TX \TX^T + n  \lambda I)^{-1}\TX X^T(XX^T)^{-1}X  ] \} \\
&  \quad + \zeta_2 \mathrm{tr}\{ [\TX^T (\TX \TX^T + n \lambda I)^{-1} \TX]^2 \varSigma \} \\
& \quad + \sigma^2 \mathrm{tr}\{ (XX^T)^{-2}X [I - \TX^T(\TX\TX^T + n \lambda I)^{-1}\TX] \varSigma [I - \TX^T(\TX\TX^T + n \lambda I)^{-1}\TX] X^T\} \\
& \quad + \tsig^2 \mathrm{tr}\{ (\TX \TX^T + n \lambda I)^{-2} \TX \varSigma \TX^T\} \\
\mathcal{L}_{\mathrm{pre}}(\hat{\theta}^{\tau}_{\lambda}) &= \theta_c^T [I - X^T (XX^T)^{-1}X][I - \tau \TX^T (\TX \TX^T + n \lambda I)^{-1} \TX] \varSigma [I - \tau \TX^T(\TX\TX^T + n \lambda I)^{-1} \TX] [I - X^T(XX^T)^{-1} X] \theta_c \\
& \quad + \zeta_1 \mathrm{tr} \{ [I - X^T (XX^T)^{-1} X + \tau X^T(XX^T)^{-1}X \TX^T(\TX \TX^T + n \lambda I)^{-1}\TX ] \varSigma \} \\
& \quad \quad \quad \quad \quad \quad [I - X^T (XX^T)^{-1} X + \tau \TX^T(\TX \TX^T + n  \lambda I)^{-1}\TX X^T(XX^T)^{-1}X  ] \} \\
&  \quad + \zeta_2 \tau^2 \mathrm{tr}\{ (\TX^T(\TX \TX^T + n \lambda I)^{-1} \TX)^2 \varSigma \} \\
& \quad + \sigma^2 \mathrm{tr}\{ (XX^T)^{-2}X [I - \tau  \TX^T(\TX\TX^T + n \lambda I)^{-1}\TX] \varSigma [I - \tau  \TX^T(\TX\TX^T + n \lambda I)^{-1}\TX] X^T\} \\
& \quad + \tau^2 \tsig^2 \mathrm{tr}\{ (\TX \TX^T + n \lambda I)^{-2} \TX \varSigma \TX^T\},
\end{align*}
and
\begin{align*}
\mathcal{L}_{\mathrm{ft}}(\hat{\theta}_1) &= \theta_c^T [I - X^T(XX^T)^{-1}X] \tvar [I - X^T(XX^T)^{-1} X] \theta_c + \zeta_1 \mathrm{tr}\{ (XX^T)^{-1} X \tvar X^T\} + \zeta_2 \mathrm{tr}\{ \tvar \}\\
& \quad + \sigma^2 \mathrm{tr}\{ (XX^T)^{-2} X \tvar X^T\}\\
\mathcal{L}_{\mathrm{ft}}(\hat{\theta}_2\} &= \theta_c^T [I - X^T (XX^T)^{-1}X][I - \TX^T (\TX \TX^T)^{-1} \TX] \tvar [I - \TX^T(\TX\TX^T)^{-1} \TX] [I - X^T(XX^T)^{-1} X] \theta_c \\
& \quad + \zeta_1 \mathrm{tr}\{ (XX^T)^{-1}X [I - \TX^T(\TX\TX^T)^{-1}\TX] \tvar [I - \TX^T(\TX\TX^T)^{-1}\TX] X^T\} + \zeta_2 \mathrm{tr}\{ (I - \TX^T(\TX \TX^T)^{-1} \TX)^2 \tvar \} \\
& \quad + \sigma^2 \mathrm{tr}\{ (XX^T)^{-2}X [I - \TX^T(\TX\TX^T)^{-1}\TX] \tvar [I - \TX^T(\TX\TX^T)^{-1}\TX] X^T\} + \tsig^2 \mathrm{tr}\{ (\TX \TX^T)^{-2} \TX \tvar \TX^T\} \\
\mathcal{L}_{\mathrm{ft}}(\hat{\theta}_\lambda) &= \theta_c^T [I - X^T (XX^T)^{-1}X][I - \TX^T (\TX \TX^T + n \lambda I)^{-1} \TX] \tvar [I - \TX^T(\TX\TX^T + n \lambda I)^{-1} \TX] [I - X^T(XX^T)^{-1} X] \theta_c \\
& \quad + \zeta_1 \mathrm{tr}\{ (XX^T)^{-1}X [I - \TX^T(\TX\TX^T + n \lambda I)^{-1}\TX] \tvar [I - \TX^T(\TX\TX^T + n \lambda I)^{-1}\TX] X^T\}\\
&  \quad + \zeta_2 \mathrm{tr}\{ (I - \TX^T(\TX \TX^T + n \lambda I)^{-1} \TX)^2 \tvar \} \\
& \quad + \sigma^2 \mathrm{tr}\{ (XX^T)^{-2}X [I - \TX^T(\TX\TX^T + n \lambda I)^{-1}\TX] \tvar [I - \TX^T(\TX\TX^T + n \lambda I)^{-1}\TX] X^T\} \\
& \quad + \tsig^2 \mathrm{tr}\{ (\TX \TX^T + n \lambda I)^{-2} \TX \tvar \TX^T\} \\
\mathcal{L}_{\mathrm{ft}}(\hat{\theta}^{\tau}_{\lambda}) &= \theta_c^T [I - X^T (XX^T)^{-1}X][I - \tau \TX^T (\TX \TX^T + n \lambda I)^{-1} \TX] \tvar [I - \tau \TX^T(\TX\TX^T + n \lambda I)^{-1} \TX] [I - X^T(XX^T)^{-1} X] \theta_c \\
& \quad + \zeta_1 \mathrm{tr}\{ (XX^T)^{-1}X [I - \tau \TX^T(\TX\TX^T + n \lambda I)^{-1}\TX] \tvar [I - \tau \TX^T(\TX\TX^T + n \lambda I)^{-1}\TX] X^T\}\\
&  \quad + \zeta_2 \mathrm{tr}\{ (I - \tau \TX^T(\TX \TX^T + n \lambda I)^{-1} \TX)^2 \tvar \} \\
& \quad + \sigma^2 \mathrm{tr}\{ (XX^T)^{-2}X [I - \tau  \TX^T(\TX\TX^T + n \lambda I)^{-1}\TX] \tvar [I - \tau  \TX^T(\TX\TX^T + n \lambda I)^{-1}\TX] X^T\} \\
& \quad + \tau^2 \tsig^2 \mathrm{tr}\{ (\TX \TX^T + n \lambda I)^{-2} \TX \tvar \TX^T\}.
\end{align*}

\subsection{Term Bounds Estimation}\label{pf:bound}
After obtaining the expressions about different terms within the excess risk, we could estimate the related upper and lower bounds now.

\subsubsection{Terms corresponding to $\theta_c$}
All of the excess risks about $\hat{\theta}_1, \hat{\theta}_2, \hat{\theta}_\lambda$ and $\hat{\theta}^{\tau}_{\lambda}$ contain a term related to $\theta_c$. Here we could obtain their upper bounds for Task 1 as
\begin{align*}
& \quad \theta_c^T [I - X^T(XX^T)^{-1}X] \varSigma [I - X^T(XX^T)^{-1} X] \theta_c \\
&=  \theta_c^T [I - X^T(XX^T)^{-1}X] \left(\varSigma - \frac{1}{n} X^TX \right) [I - X^T(XX^T)^{-1} X] \theta_c ] \le \| \theta_c \|_2^2 \| \varSigma - \frac{1}{n} X^T X \|_2 \\
& \quad \theta_c^T [I - X^T (XX^T)^{-1}X][I - \TX^T (\TX \TX^T)^{-1} \TX] \varSigma [I - \TX^T(\TX\TX^T)^{-1} \TX] [I - X^T(XX^T)^{-1} X] \theta_c\\
& \le \theta_c^T [I - X^T(XX^T)^{-1}X] \varSigma [I - X^T(XX^T)^{-1} X] \theta_c \le \| \theta_c \|_2^2 \| \varSigma - \frac{1}{n} X^T X \|_2\\
& \quad \theta_c^T [I - X^T (XX^T)^{-1}X][I - \TX^T (\TX \TX^T + n \lambda I)^{-1} \TX] \varSigma [I - \TX^T(\TX\TX^T + n \lambda I)^{-1} \TX] [I - X^T(XX^T)^{-1} X] \theta_c\\
& \le \theta_c^T [I - X^T(XX^T)^{-1}X] \varSigma [I - X^T(XX^T)^{-1} X] \theta_c \le \| \theta_c \|_2^2 \| \varSigma - \frac{1}{n} X^T X \|_2\\
& \quad \theta_c^T [I - X^T (XX^T)^{-1}X][I - \tau \TX^T (\TX \TX^T + n \lambda I)^{-1} \TX] \varSigma [I - \tau \TX^T(\TX\TX^T + n \lambda I)^{-1} \TX] [I - X^T(XX^T)^{-1} X] \theta_c\\
& \le \theta_c^T [I - X^T(XX^T)^{-1}X] \varSigma [I - X^T(XX^T)^{-1} X] \theta_c \le \| \theta_c \|_2^2 \| \varSigma - \frac{1}{n} X^T X \|_2,
\end{align*}
and for Task 2 as:
\begin{align*}
& \quad \theta_c^T [I - X^T(XX^T)^{-1}X] \tvar [I - X^T(XX^T)^{-1} X] \theta_c \\
&=  \theta_c^T [I - X^T(XX^T)^{-1}X] \left( \tvar - \frac{1}{n} X^TX \right) [I - X^T(XX^T)^{-1} X] \theta_c ] \le \| \theta_c \|_2^2 \| \tvar - \frac{1}{n} X^T X \|_2 \\
& \quad \theta_c^T [I - X^T (XX^T)^{-1}X][I - \TX^T (\TX \TX^T)^{-1} \TX] \tvar [I - \TX^T(\TX\TX^T)^{-1} \TX] [I - X^T(XX^T)^{-1} X] \theta_c\\
& \le \theta_c^T [I - X^T(XX^T)^{-1}X] \tvar [I - X^T(XX^T)^{-1} X] \theta_c \le \| \theta_c \|_2^2 \| \tvar - \frac{1}{n} X^T X \|_2\\
& \quad \theta_c^T [I - X^T (XX^T)^{-1}X][I - \TX^T (\TX \TX^T + n \lambda I)^{-1} \TX] \tvar [I - \TX^T(\TX\TX^T + n \lambda I)^{-1} \TX] [I - X^T(XX^T)^{-1} X] \theta_c\\
& \le \theta_c^T [I - X^T(XX^T)^{-1}X] \tvar [I - X^T(XX^T)^{-1} X] \theta_c \le \| \theta_c \|_2^2 \| \tvar - \frac{1}{n} X^T X \|_2\\
& \quad \theta_c^T [I - X^T (XX^T)^{-1}X][I - \tau \TX^T (\TX \TX^T + n \lambda I)^{-1} \TX] \tvar [I - \tau \TX^T(\TX\TX^T + n \lambda I)^{-1} \TX] [I - X^T(XX^T)^{-1} X] \theta_c\\
& \le \theta_c^T [I - X^T(XX^T)^{-1}X] \tvar [I - X^T(XX^T)^{-1} X] \theta_c \le \| \theta_c \|_2^2 \| \tvar - \frac{1}{n} X^T X \|_2,
\end{align*}
which implies that we just need to upper bound the following two terms
\begin{equation*}
 \| \varSigma - \frac{1}{n} X^T X \|_2, \quad \| \tvar - \frac{1}{n} X^TX \|_2 \le \| \tvar - \varSigma \|_2 + \| \varSigma - \frac{1}{n} X^T X \|_2.
\end{equation*}
From Condition~\ref{cond:eigen} and \ref{cond:order}, we have
\begin{equation*}
    \| \tvar - \varSigma \|_2 \le \max\{ \gamma, \tgam \},
\end{equation*}
and induced by Lemma~\ref{lem_eigenx}, with probability at least $1 - e^{- n^\xi}$, we have
\begin{equation}\label{eqs:theta_c}
    \| \varSigma - \frac{1}{n} X^T X \|_2 \le n^{- \frac{1 - \xi}{2}}.
\end{equation}
Combining the two results above, we can also obtain
\begin{equation}\label{eq:theta_c}
    \| \tvar - \frac{1}{n} X^T X \|_2 \le \max\{ \gamma, \tgam \} + n^{- \frac{1 - \xi}{2}}.
\end{equation}

\subsubsection{Terms corresponding to $\zeta_1$}
For these terms corresponding to $\zeta_1$ in $\mathcal{L}_{\mathrm{pre}}(\hat{\theta}_2), \mathcal{L}_{\mathrm{pre}}(\hat{\theta}_\lambda)$ and $\mathcal{L}_{\mathrm{pre}}(\hat{\theta}^{\tau}_{\lambda})$, we could approximate their upper bounds as:
\begin{align*}
& \quad \zeta_1 \mathrm{tr} \{ [I - X^T (XX^T)^{-1} X + X^T(XX^T)^{-1}X \TX^T(\TX \TX^T)^{-1}\TX ] \varSigma  \\
& \quad \quad \quad \quad \quad \quad [I - X^T (XX^T)^{-1} X +\TX^T(\TX \TX^T)^{-1}\TX X^T(XX^T)^{-1}X  ] \} \le \zeta_1 \mathrm{tr} \{ \varSigma \},\\
& \quad  \zeta_1 \mathrm{tr} \{ [I - X^T (XX^T)^{-1} X + X^T(XX^T)^{-1}X \TX^T(\TX \TX^T + n \lambda I)^{-1}\TX ] \varSigma \} \\
& \quad \quad \quad \quad \quad \quad [I - X^T (XX^T)^{-1} X +\TX^T(\TX \TX^T + n  \lambda I)^{-1}\TX X^T(XX^T)^{-1}X  ] \} \le \zeta_1 \mathrm{tr} \{ \varSigma \},\\
& \quad  \zeta_1 \mathrm{tr} \{ [I - X^T (XX^T)^{-1} X + \tau X^T(XX^T)^{-1}X \TX^T(\TX \TX^T + n \lambda I)^{-1}\TX ] \varSigma \} \\
& \quad \quad \quad \quad \quad \quad [I - X^T (XX^T)^{-1} X + \tau \TX^T(\TX \TX^T + n  \lambda I)^{-1}\TX X^T(XX^T)^{-1}X  ] \} \le \zeta_1 \mathrm{tr} \{ \varSigma \}.
\end{align*}
Similarly, for the terms in
$\mathcal{L}_{\mathrm{ft}}(\hat{\theta}_2), \mathcal{L}_{\mathrm{ft}}(\hat{\theta}_\lambda)$ and $\mathcal{L}_{\mathrm{ft}}(\hat{\theta}^{\tau}_{\lambda})$, we can obtain their upper bounds as:
\begin{align*}
& \quad \zeta_1 \mathrm{tr}\{ (XX^T)^{-1}X [I - \TX^T(\TX\TX^T)^{-1}\TX] \tvar [I - \TX^T(\TX\TX^T)^{-1}\TX] X^T\}\\
& \le \zeta_1 \mathrm{tr}\{ (XX^T)^{-1} X \tvar X^T\},\\
& \quad \zeta_1 \mathrm{tr}\{ (XX^T)^{-1}X [I - \TX^T(\TX\TX^T + n \lambda I)^{-1}\TX] \tvar [I - \TX^T(\TX\TX^T + n \lambda I)^{-1}\TX] X^T\} \\
& \le \zeta_1 \mathrm{tr}\{ (XX^T)^{-1} X \tvar X^T\},\\
& \quad \zeta_1 \mathrm{tr}\{ (XX^T)^{-1}X [I - \tau \TX^T(\TX\TX^T + n \lambda I)^{-1}\TX] \tvar [I -  \tau \TX^T(\TX\TX^T + n \lambda I)^{-1}\TX] X^T\} \\
& \le \zeta_1 \mathrm{tr}\{ (XX^T)^{-1} X \tvar X^T\},
\end{align*}
which implies that for estimating the upper bounds of these terms, we just need to upper bound the following two terms:
\begin{equation*}
 \zeta_1 \mathrm{tr}\{ \varSigma \}, \quad  \zeta_1 \mathrm{tr}\{ (XX^T)^{-1} X \tvar X^T\}.
\end{equation*}
The first term could be expressed as
\begin{equation}\label{eqs:zeta_1}
    \zeta_1 \mathrm{tr}\{ \varSigma \} = \zeta_1 \left( k^* + p \gamma \right),
\end{equation}
and we just need to approximate the second term. 
First, recalling the decomposition $\varSigma = \sum_i \lambda_i e_i e_i^T$, we have
\begin{equation}\label{eq:note1}
    XX^T = \sum_i \lambda_i z_i z_i^T, \quad X \varSigma X^T = \sum_i \lambda_i^2 z_i z_i^T,
\end{equation}
in which
\begin{equation}
  z_i := \frac1{\sqrt{\lambda_i}} X e_i
  \label{eq:z_i}
\end{equation}
are independent $\sigma_x$-subgaussian random vectors in $\mathrm{R}^n$ with mean $0$ and covariance $I$. Then we will take the following notations in further analysis: 
\begin{equation}\label{eq:note2}
    A = XX^T, \quad A_k = \sum_{i > k} \lambda_i z_i z_i^T, \quad A_{-k} = \sum_{i \ne k} \lambda_i z_i z_i^T.
\end{equation}
Using Woodbury identity, we have
\begin{equation}\label{eq:decom1}
\begin{aligned}
  \zeta_1 \mathrm{tr} \{ (XX^T)^{-1} X\tvar X^T \} 
 &= \zeta_1 \sum_{i} \tilde{\lambda}_i \lambda_i z_i^T (XX^T)^{-1} z_i \\
 &= \zeta_1 \left( \sum_{i=1}^{k^*} \frac{\tilde{\lambda}_i \lambda_i z_i^T A_{-i}^{-1} z_i}{1 + \lambda_i z_i^T A_{-i}^{-1} z_i} + \sum_{i > k^*} \tilde{\lambda}_i \lambda_i z_i^T (XX^T)^{-1} z_i \right).   
\end{aligned}
\end{equation}
For any $i = 1, \dots, k^*$, we have
\begin{equation}\label{eq:norm_up_lo}
  z_i^T A_{-i}^{-1} z_i \le \frac{\| z_i \|_2^2}{\mu_n(A_{-i})} , \quad z_i^T A_{-i}^{-1} z_i \ge (\Pi_{\mathscr{L}_i}z_i)^T A_{-i}^{-1} (\Pi_{\mathscr{L}_i} z_i) \ge \frac{\parallel \Pi_{\mathscr{L}_i} z_i \parallel_2^2}{\mu_{k^{*} + 1} (A_{-i})},  
\end{equation}
where $\mathscr{L}_i$ is denoted as the subspace in $\mathbb{R}^n$, related to the $n - k^{*}$ eigenvalues of $A_{-i}$.
Considering Lemma~\ref{lem_eigen} and Lemma~\ref{lem_subspacenorm}, with probability at least $1 - 5 e^{- n / c}$, we have
\begin{equation*}
     \frac{1}{c_1} \lambda_{k^*+1} r_{k^*} \le \mu_{n}(A_{-i}) \le \mu_{k^*+1}(A_{-i}) \le c_1 \lambda_{k^*+1} r_{k^*}, \quad \| z_i \|_2^2 \le c_2 n, \quad \| \Pi_{\mathscr{L}_i}z_i \|_2^2 \ge n / c_3,
\end{equation*}
where $c_1, c_2, c_3$ are constants only depending on $\sigma_x$. The results above imply that
\begin{equation*}
    z_i^T A_{-i}^{-1} z_i \le \frac{c_1 c_2 n}{\lambda_{k^*+1} r_{k^*} }, \quad z_i^T A_{-i}^{-1} z_i \ge \frac{n}{c_1 c_3 \lambda_{k^*+1} r_{k^*} },
\end{equation*}
so with probability at least $1 - 5 e^{- n /c}$, we have
\begin{equation}\label{eq:dcom11}
 \sum_{i=1}^{k^*} \frac{\tilde{\lambda}_i \lambda_i z_i^T A_{-1}^{-1} z_i}{1 + \lambda_i z_i^T A_{-1}^{-1} z_i} \le \sum_{i=1}^{k^*} \tilde{\lambda}_i \frac{c_1 c_2 n \lambda_i / (\lambda_{k^*+1} r_{k^*})}{1 + \lambda_i c_1 c_2 n / (\lambda_{k^*+1} r_{k^*})} \le \sum_{i=1}^{k^*} \tilde{\lambda}_i = k^*,   
\end{equation}
where the last equality is from Condition~\ref{cond:eigen}. For the remaining part, considering Lemma~\ref{lem_eigen}, with probability at least $1 - 2 e^{- n / c}$, we have
\begin{equation*}
    \sum_{i > k^*} \tilde{\lambda}_i \lambda_i z_i^T (XX^T)^{-1} z_i \le \frac{\sum_{i > k^*} \tilde{\lambda}_i \lambda_i \| z_i \|_2^2}{\mu_n(XX^T)} \le c_1^2 \frac{\sum_{i > k^*} \tilde{\lambda}_i \lambda_i \| z_i \|_2^2}{\lambda_{k^*+1} r_{k^*}} \le c_1^2 \frac{\sum_{i > k^*} \tilde{\lambda}_i \lambda_i \| z_i \|_2^2}{\lambda_{k^*+1} r_{k^*}},
\end{equation*}
 and further considering Lemma~\ref{lem_stnorm}, with probability at least $1 - 2 e^{- n / c}$, we have
\begin{align*}
    \sum_{i > k^*} \tilde{\lambda}_i \lambda_i \| z_i \|_2^2 & \le n \sum_{i > k^*} \tilde{\lambda}_i \lambda_i + 2 \sigma_x \max \left\{ \frac{n \tilde{\lambda}_{k^*+1} \lambda_{k^*+1}}{c}, \sqrt{n \sum_{i > k^*} \tilde{\lambda}_i^2 \lambda_i^2 / c} \right\}\\
    &= n \tilde{p} \tgam \gamma + 2 \sigma_x \max \left\{ \frac{n \tgam \gamma}{c}, \tgam \gamma \sqrt{\frac{n \tilde{p}}{c}} \right\} \le 2 n \tilde{p} \tgam \gamma,
\end{align*}
 where the second equality is from Condition~\ref{cond:eigen} and the last inequality is from Condition~\ref{cond:order}.
This result implies that with probability at least $ 1- 4 e^{- n / c}$, we have
\begin{equation}\label{eq:dcom12}
    \sum_{i > k^*} \tilde{\lambda}_i \lambda_i z_i^T (XX^T)^{-1} z_i \le \frac{c_1^2 2 n \tilde{p} \tgam \gamma}{\lambda_{k^*+1} r_{k^*}} = \frac{c_1^2 2 n \tilde{p} \tgam \gamma}{p \gamma} = \frac{2 c_1^2 n \tilde{p} \tgam}{p}.
\end{equation}
Combing the results in Eq.~\eqref{eq:decom1}, \eqref{eq:dcom11} and \eqref{eq:dcom12}, with probability at least $1 - 10 e^{- n /2c}$, we could obtain that
\begin{equation}\label{eq:zeta_1}
    \zeta_1 \mathrm{tr} \{ (XX^T)^{-1} X \tvar X^T \} \le \zeta_1 \left( k^* + \frac{2 c_1^2 n \tilde{p} \tgam}{p} \right).
\end{equation}

\subsubsection{Terms corresponding to $\sigma$}

Similar to the analysis above, these terms corresponding to $\sigma$ in excess risks on Task 1 and Task 2 can be upper bounded by
\begin{equation*}
 \sigma^2 \mathrm{tr}\{ (XX^T)^{-2} X \varSigma X^T \}, \quad   \sigma^2 \mathrm{tr}\{ (XX^T)^{-2} X \tvar X^T \}
\end{equation*}
respectively,
and the estimation of its upper bound is similar to the previous item. Here we first consider the second item, by Woodbury identity, we have
\begin{equation}\label{eq:decom_2}
\begin{aligned}
 \sigma^2 \mathrm{tr} \{ (XX^T)^{-2} X \tvar X^T \} &=  \sum_{i} \tilde{\lambda}_i \lambda_i z_i^T (XX^T)^{-2} z_i \\
 &= \sigma^2 \sum_{i=1}^{k^*} \frac{\tilde{\lambda}_i \lambda_i z_i^T A_{-i}^2 z_i}{[1 + \lambda_i z_i^T A_{-i}^{-1}z_i]^2} + \sigma^2 \sum_{i > k^*} \tilde{\lambda}_i \lambda_i z_i^T (XX^T)^{-2} z_i.
\end{aligned}    
\end{equation}
From the results in \eqref{eq:norm_up_lo}, with probability at least $1 - 5 e^{- n /c}$, for any $i = 1, \dots, k^*$, we have
\begin{equation*}
     z_i^T A_{-i}^{-2} z_i \le \frac{c_1^2 c_2 n}{(\lambda_{k^*+1} r_{k^*})^2 } , \quad z_i^T A_{-i}^{-1} z_i \ge \frac{n}{c_1 c_3 \lambda_{k^*+1} r_{k^*} },
\end{equation*}
which implies that
\begin{equation}\label{eq:decom_21}
\begin{aligned}
  \sigma^2 \sum_{i=1}^{k^*} \frac{\tilde{\lambda}_i \lambda_i z_i^T A_{-i}^2 z_i}{[1 + \lambda_i z_i^T A_{-i}^{-1}z_i]^2} &\le  \sigma^2 \sum_{i=1}^{k^*} \frac{\tilde{\lambda}_i \lambda_i c_1^2 c_2 n / (\lambda_{k^*+1} r_{k^*})^2 }{[1 + \lambda_i n / (c_1 c_3 \lambda_{k^*+1} r_{k^*})]^2} \\
  &\le \sigma^2 \sum_{i=1}^{k^*} \frac{\tilde{\lambda}_i \lambda_i c_1^4 c_2 c_3^2 n}{n^2 \lambda_i^2 + c_1^2 c_3^2 \lambda_{k^*+1}^2 r_{k^*}^2} \\
  & \le \sigma^2 \sum_{i=1}^{k^*} \frac{\tilde{\lambda}_i \lambda_i c_1^4 c_2 c_3^2 n}{n^2 \lambda_i^2} = \sigma^2 c_1^4 c_2 c_3^2 \frac{1}{n} \sum_{i=1}^{k^*} \frac{\tilde{\lambda}_i}{\lambda_i} = \sigma^2 c_1^4 c_2 c_3^2 \frac{k^*}{n},
\end{aligned}
\end{equation}
where the second inequality is from $(a + b)^2 \ge a^2 + b^2$ while $a, b > 0$, the third inequality is from $a^2 + b^2 \ge a^2$, and the last equality is from Condition~\ref{cond:eigen}. And for another part, considering Lemma~\ref{lem_eigen}, with probability at least $1 - 2 e^{- n / c}$, we have
 \begin{equation*}
     \sum_{i > k^*} \tilde{\lambda}_i \lambda_i z_i^T (XX^T )^{-2} z_i \le \frac{\sum_{i > k^*} \tilde{\lambda}_i \lambda_i \| z_i \|_2^2}{\mu_n(XX^T)^2 } \le c_1^2 \frac{\sum_{i > k^*} \tilde{\lambda}_i \lambda_i \| z_i \|_2^2}{\lambda_{k^*+1}^2 r_{k^*}^2 },
 \end{equation*}
 and further considering Lemma~\ref{lem_stnorm}, with probability at least $1 - 2 e^{- n / c}$, we have
\begin{align*}
    \sum_{i > k^*} \tilde{\lambda}_i \lambda_i \| z_i \|_2^2 &\le n \sum_{i > k^*} \tilde{\lambda}_i \lambda_i + 2 \sigma_x \max \left\{ \frac{n \tilde{\lambda}_{k^*+1} \lambda_{k^*+1}}{c}, \sqrt{n \sum_{i > k^*} \tilde{\lambda}_i^2 \lambda_i^2 / c} \right\}\\
    &= n \tilde{p} \tgam \gamma + 2 \sigma_x \max \left\{ \frac{n \gamma^2}{c}, \gamma^2 \sqrt{\frac{n \tilde{p}}{c}} \right\} \le 2 n \tilde{p} \gamma^2,
\end{align*}
 where the second equality is from Condition~\ref{cond:eigen} and the last inequality is from Condition~\ref{cond:order}.
 which implies that with probability at least $ 1- 4 e^{- n / c}$, we have
 \begin{equation}\label{eq:decom_22}
\sigma^2 \sum_{i > k^*} \tilde{\lambda}_i \lambda_i z_i^T (XX^T)^{-2} z_i \le \sigma^2 2 c_1^2 \frac{n  \tilde{p} \tgam \gamma}{\lambda_{k^*+1}^2 r_{k^*}^2} = \sigma^2 2 c_1^2 \frac{n \tilde{p} }{p^2 },
\end{equation}
where the last equality is from Condition~\ref{cond:eigen}. Combining the results in \eqref{eq:decom_2} \eqref{eq:decom_21} and \eqref{eq:decom_22}, with probability at least $1 - 10 e^{- n / 2c}$, we have
\begin{equation}\label{eq:sig}
    \sigma^2 \mathrm{tr} \{ (XX^T)^{-2} X \tvar X^T \} \le \sigma^2 \left( c_1^4 c_2 c_3^2 \frac{k^*}{n} + 2 c_1^2 \frac{n \tilde{p} }{p^2} \right).
\end{equation}
The analysis for the first item is similar, which implies that with probability at least $1 - 10 e^{-n/2c}$, we could obtain that
\begin{equation}\label{eqs:sig}
    \sigma^2 \mathrm{tr} \{ (XX^T)^{-2} X \varSigma X^T \} \le \sigma^2 \left( c_1^4 c_2 c_3^2 \frac{k^*}{n} + 2 c_1^2 \frac{n  }{p} \right).
\end{equation}

\subsubsection{Terms corresponding to $\zeta_2$}

In $\mathcal{L}_{\mathrm{pre}}$, to approximate the upper and lower bounds of terms related to $\zeta_2$, we need to estimate:
\begin{equation*}
    \zeta_2 \mathrm{tr} \{ [\TX^T (\TX \TX^T + n \lambda I)^{-1} \TX]^2 \Sigma \} = \zeta_2 \mathrm{tr} \{ (\TX \TX^T + n \lambda I)^{-2} \TX \TX^T \TX \varSigma \TX^T \},
\end{equation*}
for any $\lambda \ge 0$. And in $\mathcal{L}_{\mathrm{ft}}$, for  the terms related to $\zeta_2$,  
we have the following results:
\begin{align*}
 \zeta_2 \mathrm{tr} \{  (I - \TX^T (\TX\TX^T)^{-1} \TX)^2 \tvar \} &= \zeta_2 \mathrm{tr} \{ \tvar\} - \zeta_2 \mathrm{tr}\{  (\TX \TX^T)^{-1} \TX \tvar \TX^T \},\\
 \zeta_2 \mathrm{tr} \{ (I - \TX^T (\TX \TX^T + n \lambda I)^{-1}\TX)^2 \tvar \} &\le \zeta_2 \mathrm{tr} \{ \tvar (I - \TX^T(\TX\TX^T + n \lambda I)^{-1} \TX) \} \\
 &= \zeta_2 \mathrm{tr} \{ \tvar \} - \zeta_2 \mathrm{tr} \{ (\TX\TX^T + n \lambda I)^{-1} \TX \tvar \TX^T \}, \\
 \zeta_2 \mathrm{tr} \{ \tvar (I - \TX^T (\TX \TX^T + n \lambda I)^{-1} \TX)^2 \} & \ge \zeta_2 \mathrm{tr} \{ \tvar (I - \TX^T (\TX \TX^T )^{-1} \TX)^2 \}\\
 &= \zeta_2 \mathrm{tr} \{ \tvar \} - \zeta_2 \mathrm{tr} \{ (\TX\TX^T)^{-1} \TX \tvar \TX^T \},\\
\zeta_2 \mathrm{tr} \{ (I - \tau \TX^T (\TX \TX^T + n \lambda I)^{-1}\TX)^2 \tvar \} &\le \zeta_2 \mathrm{tr} \{ \tvar (I - \tau \TX^T(\TX\TX^T + n \lambda I)^{-1} \TX) \}\\
&= \zeta_2 \mathrm{tr} \{ \tvar \} - \zeta_2 \tau \mathrm{tr} \{ (\TX\TX^T + n \lambda I)^{-1} \TX \tvar \TX^T \}, \\
 \zeta_2 \mathrm{tr} \{ \tvar (I - \tau \TX^T (\TX \TX^T + n \lambda I)^{-1} \TX)^2 \} & \ge \zeta_2 \mathrm{tr} \{ \tvar (I - \tau \TX^T (\TX \TX^T )^{-1} \TX)^2 \}\\
 &= \zeta_2 \mathrm{tr} \{ \tvar \} - \zeta_2 (2 \tau - \tau^2) \mathrm{tr} \{ (\TX\TX^T)^{-1} \TX \tvar \TX^T \},
\end{align*}
so here we need to estimate the upper and lower bounds for term
\begin{equation*}
    \zeta_2 \mathrm{tr} \{ (\TX \TX^T + n \lambda I)^{-1} \TX \tvar \TX^T \}
\end{equation*}
for any $\lambda \ge 0$. The analysis is similar to the analysis above, recalling the decomposition $\tvar = \sum_i \tilde{\lambda}_i e_i e_i^T$, we have
\begin{equation*}
    \TX \TX^T = \sum_i \tilde{\lambda}_i \tz_i \tz_i^T, \quad \TX \tvar \TX^T = \sum_i \tilde{\lambda}_i^2 \tz_i \tz_i^T,
\end{equation*}
in which 
\begin{equation*}
    \tz := \frac{1}{\sqrt{\tilde{\lambda}_i}} \TX e_i,
\end{equation*}
are independent $\sigma_x$-subgaussian random vectors in $\mathbb{R}^n$ with mean zero and covariance $I$. So we take the following notations in further analysis:
\begin{equation}\label{eq:note3}
    \tilde{A} = \TX \TX^T, \quad \tilde{A}_k  = \sum_{i > k} \tilde{\lambda}_i \tz_i \tz_i^T, \quad \tilde{A}_{-k} = \sum_{i \ne k} \tilde{\lambda}_i \tz_i \tz_i^T
\end{equation}
Starting with the analysis of $\zeta_2\mathrm{tr} \{ (\TX \TX^T + n \lambda I)^{-1} \TX \tvar \TX^T \}$,
we consider its upper bound firstly. Using Woodbury identity, we have
\begin{equation}\label{eq:decom3}
\begin{aligned}
& \quad \zeta_2 \mathrm{tr}\{ (\TX \TX^T + n \lambda I)^{-1} \TX \tvar \TX^T \} \\
&= \zeta_2 \sum_i \tilde{\lambda}_i^2 \tz_i^T (\TX\TX^T + n \lambda I)^{-1} \tz_i\\
&= \zeta_2 \left( \sum_{i=1}^{k^*} \frac{\tilde{\lambda}_i^2 \tz_i^T (\tilde{A}_{-i} + n \lambda I)^{-1} \tz_i}{1 + \tilde{\lambda}_i \tz_i^T (\tilde{A}_{-i} + n \lambda I)^{-1} \tz_i} + \sum_{i > k^*} \tilde{\lambda}_i^2 \tz_i^T (\TX\TX^T + n \lambda I)^{-1} \tz_i \right).
\end{aligned}
\end{equation}
For any $i = 1, \dots, k^*$, we have
\begin{equation}\label{eq:tnorm_up_lo}
\begin{aligned}
& \tz_i^T (\tilde{A}_{-i} + n \lambda I)^{-1} \tz_i \le \frac{\| \tz_i \|_2^2}{\mu_n(\tilde{A}_{-i}) + n \lambda} ,\\
& \tz_i^T ( \tilde{A}_{-i} + n \lambda I )^{-1} \tz_i \ge (\Pi_{\tilde{\mathscr{L}}_i} \tz_i)^T ( \tilde{A}_{-i} + n \lambda I)^{-1} (\Pi_{\tilde{\mathscr{L}}_i} \tz_i) \ge \frac{\parallel \Pi_{\tilde{\mathscr{L}}_i} z_i \parallel_2^2}{\mu_{k^{*} + 1} (\tilde{A}_{-i})},     
\end{aligned}
\end{equation}
where $\tilde{\mathscr{L}}_i$ is denoted as the subspace in $\mathbb{R}^n$, related to the $n - k^{*}$ eigenvalues of $\tilde{A}_{-i}$.
Considering Lemma~\ref{lem_eigen} and Lemma~\ref{lem_subspacenorm}, with probability at least $1 - 5 e^{- n / c}$, we have
\begin{equation}\label{eq:tnorm2}
     \frac{1}{c_1} \tilde{\lambda}_{k^*+1} \tilde{r}_{k^*} \le \mu_{n}(\tilde{A}_{-i}) \le \mu_{k^*+1}(\tilde{A}_{-i}) \le c_1 \tilde{\lambda}_{k^*+1} \tilde{r}_{k^*}, \quad \| \tz_i \|_2^2 \le c_2 n, \quad \| \Pi_{\tilde{\mathscr{L}}_i} \tz_i \|_2^2 \ge n / c_3,
\end{equation}
where $c_1, c_2, c_3$ are constants only depending on $b, \sigma_x$. The results above imply that
\begin{equation*}
    \tz_i^T (\tilde{A}_{-i} + n \lambda I)^{-1} \tz_i \le \frac{c_1 c_2 n}{\tilde{\lambda}_{k^*+1} \tilde{r}_{k^*} + n \lambda }, \quad \tz_i^T (\tilde{A}_{-i} + n \lambda )^{-1} z_i \ge \frac{n}{c_1 c_3 ( \tilde{\lambda}_{k^*+1} \tilde{r}_{k^*} + n \lambda) }.
\end{equation*}
so with probability at least $1 - 5 e^{- n /c}$, we have
\begin{equation}\label{eq:dcom31}
\begin{aligned}
\sum_{i=1}^{k^*} \frac{\tilde{\lambda}_i^2  \tz_i^T (\tilde{A}_{-i} + n \lambda I)^{-1} \tz_i}{1 + \tilde{\lambda}_i \tz_i^T (\tilde{A}_{-i} + n \lambda I )^{-1} \tz_i} &\le \sum_{i=1}^{k^*} \tilde{\lambda}_i \frac{c_1 c_2 n \tilde{\lambda}_i / (\tilde{\lambda}_{k^*+1} \tilde{r}_{k^*} + n \lambda)}{1 + \tilde{\lambda}_i c_1 c_2 n / (\tilde{\lambda}_{k^*+1} \tilde{r}_{k^*} + n \lambda)}\\
&= k^* \frac{c_1 c_2 n / (\tilde{p} \tgam + n \lambda)}{1 + c_1 c_2 n / (\tilde{p} \tgam + n \lambda)}\\
&\le k^* \min\{ \frac{c_1 c_2 n}{c_1 c_2 n + \tilde{p} \tgam}, \frac{c_1 c_2}{c_1 c_2 + \lambda} \},
\end{aligned}
\end{equation}
where the last inequality is from $a + b \ge \max\{ a, b\}$.
For the remaining part, considering Lemma~\ref{lem_eigen}, with probability at least $1 - 2 e^{- n / c}$, we have
\begin{equation*}
    \sum_{i > k^*} \tilde{\lambda}_i^2  \tz_i^T (\TX \TX^T + n \lambda I)^{-1} \tz_i \le \frac{\sum_{i > k^*} \tilde{\lambda}_i^2  \| \tz_i \|_2^2}{\mu_n( \TX \TX^T) + n \lambda} \le c_1^2 \frac{\sum_{i > k^*} \tilde{\lambda}_i^2  \| \tz_i \|_2^2}{\tilde{\lambda}_{k^*+1} \tilde{r}_{k^*} + n \lambda},
\end{equation*}
and further considering Lemma~\ref{lem_stnorm}, with probability at least $1 - 2 e^{- n / c}$, we have
\begin{align*}
   \sum_{i > k^*} \tilde{\lambda}_i^2 \| \tz_i \|_2^2 &\le n \sum_{i > k^*} \tilde{\lambda}_i^2 + 2 \sigma_x \max \left\{ \frac{n \tilde{\lambda}_{k^*+1}^2}{c}, \sqrt{n \sum_{i > k^*} \tilde{\lambda}_i^4 / c} \right\} \\
   &= n \tilde{p} \tgam^2 + 2 \sigma_x \max \left\{ \frac{n \tgam^2}{c}, \tgam^2 \sqrt{\frac{n \tilde{p}}{c}} \right\} \le 2 n \tilde{p} \tgam^2,  
\end{align*}
 where the second equality is from Condition~\ref{cond:eigen} and the last inequality is from Condition~\ref{cond:order}.
This result implies that with probability at least $ 1- 4 e^{- n / c}$, we have
\begin{equation}\label{eq:dcom32}
    \sum_{i > k^*} \tilde{\lambda}_i^2 \tz_i^T (\TX \TX^T + n \lambda I)^{-1} \tz_i \le \frac{c_1^2 2 n \tilde{p} \tgam^2}{\tilde{\lambda}_{k^*+1} \tilde{r}_{k^*} + n \lambda} = \frac{2 c_1^2 n \tilde{p} \tgam^2}{\tilde{p} \tgam + n \lambda} \le 2 c_1^2 \min \left\{  n \tgam, \frac{\tilde{p} \tgam^2}{\lambda}  \right\}.
\end{equation}
Combing the results in Eq.~\eqref{eq:decom3}, \eqref{eq:dcom31} and \eqref{eq:dcom32}, with probability at least $1 - 10 e^{- n /2c}$, we could obtain that
\begin{equation}\label{eq:zeta_2_up}
    \zeta_2 \mathrm{tr} \{ (\TX \TX^T + n \lambda I)^{-1} \TX \tvar \TX^T \} \le \zeta_2 \left( k^* \min\{ \frac{c_1 c_2 n}{c_1 c_2 n + \tilde{p} \tgam}, \frac{c_1 c_2}{c_1 c_2 + \lambda} \} + 2 c_1^2 \min \left\{  n \tgam, \frac{\tilde{p} \tgam^2}{\lambda}  \right\} \right).
\end{equation}
Then we turn to its lower bound. Considering \eqref{eq:tnorm_up_lo} for any index $i = 1, \dots, \infty$, 
as
\begin{equation*}
    \mu_{k^*+1} (\tilde{A}_{-i}) \le \mu_{k^*+1}(\TX \TX^T) \le c_1 \tilde{\lambda}_{k^*+1} \tilde{r}_{k^*}
\end{equation*}
is always satisfied,
with probability at least $1 - 5 e^{- n / c}$ we can get a lower bound as
\begin{equation}\label{eq:dcom33}
\begin{aligned}
\frac{\tilde{\lambda}_i^2 \tx_i^T (\tilde{A}_{-i} + n \lambda I)^{-1} \tz_i}{1 + \tilde{\lambda}_i \tz_i^T (\tilde{A}_{-i} + n \lambda I )^{-1} \tz_i} &\ge \frac{\tilde{\lambda}_i^2 n / (\tilde{\lambda}_{k^*+1} \tilde{r}_{k^*} + n \lambda)}{c_1 c_3 + \tilde{\lambda}_i n / (\tilde{\lambda}_{k^*+1} \tilde{r}_{k^*} + n \lambda) } \\
&\ge \frac{1}{c_1 c_3} \frac{\tilde{\lambda}_i^2 n / (\tilde{\lambda}_{k^*+1} \tilde{r}_{k^*} + n \lambda)}{1 + \tilde{\lambda}_i n / (\tilde{\lambda}_{k^*+1} \tilde{r}_{k^*} + n \lambda)} > 0,
\end{aligned}
\end{equation}
and for the whole trace term $\zeta_2 \mathrm{tr} \{ (\TX \TX^T + n \lambda I)^{-1} \TX \tvar \TX^T \}$, due to Lemma \ref{lem_sum}, with probability at least $1 - 10 e^{- n / c }$, we have
\begin{equation}\label{eq:zeta2_lo}
 \begin{aligned}
   \zeta_2 \mathrm{tr} \{ (\TX \TX^T + n \lambda I)^{-1} \TX \tvar \TX^T \} 
  & = \zeta_2 \sum_i \frac{\tilde{\lambda}_i^2 \tx_i^T (\tilde{A}_{-i} + n \lambda I)^{-1} \tz_i}{1 + \tilde{\lambda}_i \tz_i^T (\tilde{A}_{-i} + n \lambda I )^{-1} \tz_i} \\
  &\ge \frac{\zeta_2}{2 c_1 c_3} \sum_i \frac{\tilde{\lambda}_i^2 n / (\tilde{\lambda}_{k^*+1} \tilde{r}_{k^*} + n \lambda)}{1 + \tilde{\lambda}_i n / (\tilde{\lambda}_{k^*+1} \tilde{r}_{k^*} + n \lambda)} \\
  &\ge \frac{\zeta_2}{6 c_1 c_3} \sum_i \min \left\{ \frac{\tilde{\lambda}_i^2}{\lambda}, \tilde{\lambda}_i, \frac{n \tilde{\lambda}_i^2}{\tilde{\lambda}_{k^*+1} r_{k^*}} \right\} \\
  &= \frac{\zeta_2}{6 c_1 c_3} \left( k^* \min\{ \frac{1}{\lambda}, 1 \} + (\tilde{p} - k^*) \min\{ \frac{\tgam^2}{\lambda}, \frac{n \tgam}{\tilde{p}} \} \right),
 \end{aligned}   
\end{equation}
where the first equality is from Woodbury identity, the first inequality is from \eqref{eq:dcom33}, the second inequality is from the fact
\begin{equation*}
  (a + b + c)^{-1} \ge (3 \max\{ a, b, c \})^{-1} = \frac{1}{3} \min \{ a^{-1}, b^{-1}, c^{-1} \}, \quad a, b, c > 0,  
\end{equation*}
and the last equality is induced from Condition~\ref{cond:eigen} and Condition~\ref{cond:order}.
And combing \eqref{eq:zeta_2_up} and \eqref{eq:zeta2_lo}, we could consider different situations with respect to the value of $\lambda$. 

While $\lambda \le \tilde{p} \tgam / n$, we have
\begin{equation}\label{eq:zeta2_s}
\begin{aligned}
\zeta_2 \mathrm{tr} \{ (\TX \TX^T + n \lambda I)^{-1} \TX \tvar \TX^T \} &\le  \zeta_2 \left( \frac{c_1 c_2 k^* n}{c_1 c_2 n + \tilde{p} \tgam} + 2 c_1^2 n \tgam \right) ,\\
\zeta_2 \mathrm{tr} \{ (\TX \TX^T + n \lambda I)^{-1} \TX \tvar \TX^T \} &\ge \frac{\zeta_2}{6 c_1 c_3} \left( k^* + \frac{n \tgam (\tilde{p} - k^*)}{\tilde{p}} \right).
\end{aligned}    
\end{equation}
And while $\lambda \ge \tilde{p} \tgam / n$, with a high probability, we have
\begin{equation}\label{eq:zeta2_l}
 \begin{aligned}
\zeta_2 \mathrm{tr} \{ (\TX \TX^T + n \lambda I)^{-1} \TX \tvar \TX^T \} &\le  \zeta_2 \left( \frac{c_1 c_2 k^*}{c_1 c_2  + \lambda} + \frac{2 c_1^2 \tilde{p} \tgam^2}{\lambda} \right) ,\\
\zeta_2 \mathrm{tr} \{ (\TX \TX^T + n \lambda I)^{-1} \TX \tvar \TX^T \} &\ge \frac{\zeta_2}{6 c_1 c_3} \left( k^* \min\{ \frac{1}{\lambda}, 1\} + \frac{ \tgam^2 (\tilde{p} - k^*)}{\lambda} \right).   
 \end{aligned}   
\end{equation}
Then we turn to the term $\zeta_2 \mathrm{tr}\{ (\TX \TX^T + n \lambda I)^{-2} \TX \TX^T \TX \varSigma \TX^T\}$. While we have $ n \lambda \le \tilde{p} \tgam / c_1$, with probability at least $1 - 5 e^{- n / c}$, we have
\begin{equation*}
    n \lambda \le \mu_n(\TX \TX^T) \Longrightarrow n \lambda I \preceq \TX \TX^T \Longleftarrow \TX\TX^T + n \lambda I \preceq 2 \TX \TX^T,
\end{equation*}
which implies that
\begin{equation}\label{eqs:z21}
    \zeta_2 \mathrm{tr} \{ (\TX \TX^T + n \lambda I)^{-2} \TX \TX^T \TX \tvar \TX^T \} \ge \frac{\zeta_2}{4} \mathrm{tr} \{ (\TX\TX^T)^{-1} \TX \varSigma \TX^T \}, 
\end{equation}
as we also have
\begin{equation}\label{eqs:z22}
    \TX\TX^T +  \lambda I \succeq \TX \TX^T \Longrightarrow \zeta_2 \mathrm{tr} \{ (\TX \TX^T + n \lambda I)^{-2} \TX \TX^T \TX \tvar \TX^T \} \le \zeta_2 \mathrm{tr} \{ (\TX\TX^T)^{-1} \TX \varSigma \TX^T \}.
\end{equation}
Combing both of the two results above, we just need to estimate the term $\zeta_2 \mathrm{tr} \{ (\TX \TX^T)^{-1} \TX \varSigma \TX^T \}$. Using Woodbury identity, we have
\begin{equation*}
   \zeta_2 \mathrm{tr}\{ (\TX \TX^T + n \lambda I)^{-1} \TX \varSigma \TX^T \} = \zeta_2 \sum_i \tilde{\lambda}_i \lambda_i \tz_i^T (\TX\TX^T + n \lambda I)^{-1} \tz_i,
\end{equation*}
and recalling the identity
\begin{equation*}
    \zeta_2 \mathrm{tr}\{ (\TX \TX^T + n \lambda I)^{-1} \TX \tvar \TX^T \} = \zeta_2 \sum_i \tilde{\lambda}_i^2 \tz_i^T (\TX\TX^T + n \lambda I)^{-1} \tz_i,
\end{equation*}
comparing such two terms and Condition~\ref{cond:eigen}, we have
\begin{equation*}
    \tilde{\lambda}_i \lambda_i = \tilde{\lambda}_i^2, \quad \forall i = 1, \dots, p,
\end{equation*}
which implies that
\begin{equation}\label{eq:eq1}
   \zeta_2 \mathrm{tr}\{ (\TX \TX^T + n \lambda I)^{-1} \TX \varSigma \TX^T \} = \zeta_2 \mathrm{tr}\{ (\TX \TX^T + n \lambda I)^{-1} \TX \tvar \TX^T \}.
\end{equation}
Then considering \eqref{eq:zeta2_s}, \eqref{eqs:z21} and \eqref{eqs:z22}, while $n \lambda \le \tilde{p} \tgam / c_1$, we have
\begin{equation}\label{eqs:zeta2}
\begin{aligned}
 \zeta_2 \mathrm{tr} \{ (\TX \TX^T + n \lambda I)^{-2} \TX \TX^T \TX \tvar \TX^T \} &\le  \zeta_2 \left( \frac{c_1 c_2 k^* n}{c_1 c_2 n + \tilde{p} \tgam} + 2 c_1^2 n \tgam \right) ,\\
 \zeta_2 \mathrm{tr} \{ (\TX \TX^T + n \lambda I)^{-2} \TX \TX^T \TX \tvar \TX^T \} &\ge \frac{\zeta_2}{24 c_1 c_3} \left( k^* + \frac{n \tgam (\tilde{p} - k^*)}{\tilde{p}} \right).
\end{aligned}    
\end{equation}

\subsubsection{Terms corresponding to $\tsig$}

Finally, for the terms related to $\tsig$, we need to estimate the bounds for the terms $\tsig^2 \mathrm{tr}\{ (\TX \TX^T)^{-2} \TX^T \varSigma \TX\}$ and $\tsig^2 \mathrm{tr} \{ (\TX \TX^T + n \lambda I)^{-2} \TX^T \varSigma \TX \}$ in $\mathcal{L}_{\mathrm{pre}}$, and 
the terms $\tsig^2 \mathrm{tr}\{ (\TX \TX^T)^{-2} \TX^T \tvar \TX\}$ and $\tsig^2 \mathrm{tr} \{ (\TX \TX^T + n \lambda I)^{-2} \TX^T \tvar \TX \}$ in $\mathcal{L}_{\mathrm{ft}}$, i.e, for the terms $\tsig^2 \mathrm{tr} \{ (\TX\TX^T + n \lambda I)^{-2} \TX \varSigma \TX^T \}$ and $\tsig^2 \mathrm{tr} \{ (\TX\TX^T + n \lambda I)^{-2} \TX \tvar \TX^T \}$ in which $\lambda \ge 0$. 

Firstly, we consider the approximation on $\tsig^2 \mathrm{tr} \{ (\TX\TX^T + n \lambda I)^{-2} \TX \tvar \TX^T \}$. For its upper bound, we could take Woodbury identity as follows:
\begin{equation}\label{eq:decom41}
\begin{aligned}
& \quad \tsig^2 \mathrm{tr}\{ (\TX \TX^T + n \lambda I)^{-2} \TX \tvar \TX^T \}\\
&= \tsig^2  \sum_i \tilde{\lambda}_i^2 \tz_i^T (\TX \TX^T + n \lambda I)^{-2} \tz_i\\
&= \tsig^2 \sum_{i=1}^{k^*} \frac{\tilde{\lambda}_i^2 \tz_i^T (\tilde{A}_{-i} + n \lambda I)^{-2} \tz_i}{[1 + \tilde{\lambda}_i \tz_i^T (\tilde{A}_{-i} + n \lambda I)^{-2} \tz_i]^2} + \tsig^2 \sum_{i > k^*} \tilde{\lambda}_i^2 \tz_i^T (\TX \TX^T + n \lambda I)^{-2} \tz_i.
\end{aligned}   
\end{equation}
According to \eqref{eq:tnorm_up_lo} and \eqref{eq:tnorm2}, with probability at least $1 - 5 e^{-n/c}$, for any index $i = 1, \dots, k^*$, we have
\begin{equation}\label{eq:decom4}
    \tz_i^T (\tilde{A}_{-i} + n \lambda I)^{-2} \tz_i \le \frac{c-1^2 c_2 n}{(\tilde{\lambda}_{k^*+1} \tilde{r}_{k^*} + n \lambda)^2}, \quad \tz_i (\tilde{A}_{-i} + n \lambda I)^{-1} \tz_i \ge \frac{n}{c_1 c_3 (\tilde{\lambda}_{k^*+1} \tilde{r}_{k^*} + n \lambda)},
\end{equation}
which implies that
\begin{equation}\label{eq:decom42}
 \begin{aligned}
  \tsig^2 \sum_{i=1}^{k^*} \frac{\tilde{\lambda}_i^2  \tz_i^T (\tilde{A}_{-i} + n \lambda I)^{-2} \tz_i}{[1 + \tilde{\lambda}_i \tz_i^T (\tilde{A}_{-i} + n \lambda I)^{-1} \tz_i]^2} &\le  \tsig^2 \sum_{i=1}^{k^*} \frac{\tilde{\lambda}_i^2  c_1^2 c_2 n / (\tilde{\lambda}_{k^*+1} \tilde{r}_{k^*} + n \lambda)^2 }{[1 + \tilde{\lambda}_i n / (c_1 c_3 \tilde{\lambda}_{k^*+1} \tilde{r}_{k^*} + c_1 c_3 n \lambda)]^2} \\
  &\le \tsig^2 \sum_{i=1}^{k^*} \frac{\tilde{\lambda}_i^2 c_1^4 c_2 c_3^2 n}{n^2 \tilde{\lambda}_i^2 + c_1^2 c_3^2 \tilde{\lambda}_{k^*+1}^2 \tilde{r}_{k^*}^2 + c_1^2 c_3^2 n^2 \lambda^2}\\
  &\le \tsig^2 c_1^4 c_2 c_3^2 \sum_{i=1}^{k^*} \min \left\{ \frac{1}{n}, \frac{n \tilde{\lambda}_i^2}{\tilde{\lambda}_{k^*+1}^2 \tilde{r}_{k^*}^2}, \frac{\tilde{\lambda}_i^2}{\lambda^2} \right\} \\
  &= \tsig^2 c_1^4 c_2 c_3^2 k^* \min \left\{ \frac{1}{n}, \frac{1}{\lambda^2} \right\},
 \end{aligned}   
\end{equation}
where the third inequality is from $a^2 + b^2 + c^2 \ge \max\{ a^2, b^2, c^2 \}$, and the last equality is from Condition~\ref{cond:eigen} and Condition~\ref{cond:order}. As for the remaining part, considering Lemma~\ref{lem_eigen}, with probability at least $1 - 2 e^{- n / c}$, we have
 \begin{equation*}
     \sum_{i > k^*} \tilde{\lambda}_i^2 \tz_i^T (\TX \TX^T + n \lambda I )^{-2} \tz_i \le \frac{\sum_{i > k^*} \tilde{\lambda}_i^2 \| \tz_i \|_2^2}{\mu_n(\TX \TX^T)^2 + n^2 \lambda^2 } \le c_1^2 \frac{\sum_{i > k^*} \tilde{\lambda}_i^2 \| \tz_i \|_2^2}{\tilde{\lambda}_{k^*+1}^2 \tilde{r}_{k^*}^2 + n^2 \lambda^2 } = c_1^2 \frac{\sum_{i > k^*} \tilde{\lambda}_i^2 \| \tz_i \|_2^2}{\tilde{p}^2 \tgam^2 + n^2 \lambda^2 },
 \end{equation*}
 and further considering Lemma~\ref{lem_stnorm}, with probability at least $1 - 2 e^{- n / c}$, we have
 \begin{align*}
     \sum_{i > k^*} \tilde{\lambda}_i^2 \| \tz_i \|_2^2 & \le n \sum_{i > k^*} \tilde{\lambda}_i^2 + 2 \sigma_x \max \left\{ \frac{n \tilde{\lambda}_{k^*+1}^2 }{c}, \sqrt{n \sum_{i > k^*} \tilde{\lambda}_i^4  / c} \right\} \\
     &= n \tilde{p} \tgam^2 + 2 \sigma_x \max \left\{ \frac{n \tgam^2 }{c}, \tgam^2  \sqrt{\frac{n \tilde{p}}{c}} \right\} \le 2 n \tilde{p} \tgam^2,  
 \end{align*}
 where the second inequality is from Condition~\ref{cond:eigen} and the last inequality is from Condition~\ref{cond:order}. It implies that with probability at least $1 - 4 e^{- n / c}$, we have
 \begin{equation}\label{eq:decom43}
     \tsig^2 \sum_{i > k^*} \tilde{\lambda}_i^2 \tz_i^T (\TX \TX^T + n \lambda I)^{-2} \tz_i \le \tsig^2  2 c_1^2 \frac{n \tilde{p} \tgam^2}{\tilde{p}^2 \tgam^2 + n^2 \lambda^2} \le 2 c_1^2 \tsig^2 \min \left\{ \frac{n}{\tilde{p}}, \frac{\tilde{p} \tgam^2}{n \lambda^2} \right\}.
 \end{equation}
Combing the results in \eqref{eq:decom41}, \eqref{eq:decom42} and \eqref{eq:decom43}, with probability at least $1 - 10 e^{- n / 2c}$, we have
\begin{equation}\label{eq:tsig_up}
    \tsig^2 \mathrm{tr} \{ (\TX \TX^T + n \lambda I)^{-2} \le \tsig^2 \left( c_1^4 c_2 c_3^2 k^* \min \left\{ \frac{1}{n}, \frac{1}{\lambda^2} \right\} + 2 c_1^2 \min \left\{ \frac{n}{\tilde{p}}, \frac{\tilde{p} \tgam^2}{n \lambda^2} \right\} \right).
\end{equation}
Then for the lower bound, on each index $i = 1, \dots, \infty$, we have
\begin{equation*}
    \tz_i^T (\tilde{A}_{-i} + n \lambda I)^{-2} \tz_i \ge \frac{1}{\| \tz_i \|_2^2} \left( \tz_i^T (\tilde{A}_{-i} + n \lambda I )^{-1} \tz_i \right)^2,
\end{equation*}
which implies that with probability at least $1 - 5 e^{-n/c}$,
\begin{equation}\label{eq:decom44}
\begin{aligned}
  \frac{\tilde{\lambda}_i^2 \tz_i^T (\tilde{A}_{-i} + n \lambda I)^{-2} \tz_i}{[1 + \tilde{\lambda}_i \tz_i^T (\tilde{A}_{-i} + n \lambda I)^{-2} \tz_i]^2} &\ge \frac{1}{\| \tz_i \|_2^2} \left( \frac{\tilde{\lambda}_i \tz_i^T (\tilde{A}_{-i} + n \lambda I)^{-1} \tz_i}{1 + \tilde{\lambda}_i \tz_i^T (\tilde{A}_{-i} + n \lambda I)^{-1} \tz_i} \right)^2 \\
  &\ge \frac{1}{c_2 n} \left( \frac{n \tilde{\lambda}_i}{c_1 c_3 (\tilde{\lambda}_{k^*+1} \tilde{r}_{k^*}) + n \tilde{\lambda}_i} \right)^2 > 0,
\end{aligned}
\end{equation}
where the last inequality is from \eqref{eq:tnorm_up_lo} and \eqref{eq:decom4}. Then for the whole term, due to Lemma~\ref{lem_sum}, with probability at least $1 - 10 e^{- n / c}$, we have
\begin{equation}\label{eq:tsig_lo}
\begin{aligned}
 \tsig^2 \mathrm{tr}\{ (\TX \TX^T + n \lambda I)^{-2} \TX \tvar \TX^T \} &= \tsig^2 \sum_i \frac{\tilde{\lambda}_i^2 \tz_i^T (\tilde{A}_{-i} + n \lambda I)^{-2} \tz_i}{[1 + \tilde{\lambda}_i \tz_i^T (\tilde{A}_{-i} + n \lambda I)^{-2} \tz_i]^2} \\
 &\ge \frac{\tsig^2}{2 c_2 n} \sum_i \left( \frac{n \tilde{\lambda}_i}{c_1 c_3 (\tilde{\lambda}_{k^*+1} \tilde{r}_{k^*} + n \lambda) + n \tilde{\lambda}_i} \right)^2\\
 &\ge \frac{\tsig^2}{18 c_1^2 c_2 c_3^2 n} \sum_i \min \left\{ 1, \frac{\tilde{\lambda}_i^2}{\lambda^2}, \frac{n^2 \tilde{\lambda}_i^2}{\tilde{\lambda}_{k^*+1}^2 \tilde{r}_{k^*}^2} \right\}\\
 &= \frac{\tsig^2}{18 c_1^2 c_2 c_3^2 n} \left( k^* \min\{ 1, \frac{1}{\lambda^2}, \frac{n^2}{\tilde{p}^2 \tgam^2} \} + (\tilde{p} - k^*) \min\{ 1, \frac{\tgam^2}{\lambda^2}, \frac{n^2}{\tilde{p}^2} \} \right)\\
 &= \frac{\tsig^2}{18 c_1^2 c_2 c_3^2 n} \left( k^* \min\{ 1, \frac{1}{\lambda^2} \} + (\tilde{p} - k^*) \min\{  \frac{\tgam^2}{\lambda^2}, \frac{n^2}{\tilde{p}^2} \} \right),
\end{aligned}    
\end{equation}
 where the first inequality is from \eqref{eq:decom44}, the second inequality is from the fact that
 \begin{equation*}
     (a + b + c)^{-2} \ge (3 \max\{ a, b, c \})^{-2} = \frac{1}{9} \min \{ a^{-2}, b^{-2}, c^{-2} \}, \quad \forall a, b, c > 0, 
 \end{equation*}
and the last two equality is from Condition~\ref{cond:eigen} and Condition~\ref{cond:order}. Combing both \eqref{eq:tsig_up} and \eqref{eq:tsig_lo}, we could also consider the following two cases.

While $\lambda \le \tilde{p} \tgam / n$, with a high probability, we have
\begin{equation}\label{eq:tsig_s}
\begin{aligned}
\tsig^2 \mathrm{tr} \{ (\TX \TX^T + n \lambda I)^{-2} \TX \tvar \TX^T \} &\le \tsig^2 \left( \frac{c_1^4 c_2 c_3^2 k^*}{n} + 2 c_1^2 \frac{n}{\tilde{p}} \right),\\
\tsig^2 \mathrm{tr} \{ (\TX \TX^T + n \lambda I)^{-2} \TX \tvar \TX^T \} &\ge \tsig^2 \frac{1}{18 c_1^2 c_2 c_3^2} \left( \frac{k^*}{n} + \frac{n (\tilde{p} - k^*)}{\tilde{p}^2} \right).    
\end{aligned}
\end{equation}

And while $\lambda \ge \tilde{p} \tgam / n$, we have
\begin{equation}\label{eq:tsig_l}
 \begin{aligned}
\tsig^2 \mathrm{tr} \{ (\TX \TX^T + n \lambda I)^{-2} \TX \tvar \TX^T \} &\le \tsig^2 \left( c_1^4 c_2 c_3^2 k^*  \min\{ \frac{1}{n}, \frac{1}{\lambda^2} \} + 2 c_1^2 \frac{\tilde{p}^2 \tgam^2}{n \lambda^2} \right),\\
\tsig^2 \mathrm{tr} \{ (\TX \TX^T + n \lambda I)^{-2} \TX \tvar \TX^T \} &\ge \tsig^2 \frac{1}{18 c_1^2 c_2 c_3^2} \left( \frac{k^*}{n} \min \{ 1, \frac{1}{\lambda^2} \} + \frac{\tgam^2 (\tilde{p} - k^*)}{n \lambda^2} \right).    
\end{aligned}  
\end{equation}

Then we turn to the term $\tsig^2 \mathrm{tr} \{ (\TX \TX^T + n \lambda I)^{-2} \TX \varSigma \TX^T \}$. Similarly, using Woodbury identity, we have
\begin{equation*}
  \tsig^2 \mathrm{tr}\{ (\TX \TX^T + n \lambda I)^{-2} \TX \varSigma \TX^T \} = \tsig^2 \sum_i \tilde{\lambda}_i \lambda_i \tz_i^T (\TX\TX^T + n \lambda I)^{-2} \tz_i,
\end{equation*}
as well as
\begin{equation*}
    \tsig^2 \mathrm{tr}\{ (\TX \TX^T + n \lambda I)^{-2} \TX \tvar \TX^T \} = \tsig^2 \sum_i \tilde{\lambda}_i^2 \tz_i^T (\TX\TX^T + n \lambda I)^{-2} \tz_i,
\end{equation*}
comparing such two terms and Condition~\ref{cond:eigen}, we have
\begin{equation*}
    \tilde{\lambda}_i \lambda_i = \tilde{\lambda}_i^2, \quad \forall i = 1, \dots, p,
\end{equation*}
which implies that
\begin{equation}\label{eq:eq2}
   \tsig^2 \mathrm{tr}\{ (\TX \TX^T + n \lambda I)^{-2} \TX \varSigma \TX^T \} = \tsig^2 \mathrm{tr}\{ (\TX \TX^T + n \lambda I)^{-2} \TX \tvar \TX^T \}.
\end{equation}
The result above implies that while $\lambda \le \tilde{p} \tgam / n$, with a high probability, we have
\begin{equation}\label{eqs:tsig_s}
\begin{aligned}
\tsig^2 \mathrm{tr} \{ (\TX \TX^T + n \lambda I)^{-2} \TX \varSigma \TX^T \} &\le \tsig^2 \left( \frac{c_1^4 c_2 c_3^2 k^*}{n} + 2 c_1^2 \frac{n}{\tilde{p}} \right),\\
\tsig^2 \mathrm{tr} \{ (\TX \TX^T + n \lambda I)^{-2} \TX \varSigma \TX^T \} &\ge \tsig^2 \frac{1}{18 c_1^2 c_2 c_3^2} \left( \frac{k^*}{n} + \frac{n (\tilde{p} - k^*)}{\tilde{p}^2} \right).    
\end{aligned}
\end{equation}
And while $\lambda \ge \tilde{p} \tgam / n$, we have
\begin{equation}\label{eqs:tsig_l}
 \begin{aligned}
\tsig^2 \mathrm{tr} \{ (\TX \TX^T + n \lambda I)^{-2} \TX \varSigma \TX^T \} &\le \tsig^2 \left( c_1^4 c_2 c_3^2 k^*  \min\{ \frac{1}{n}, \frac{1}{\lambda^2} \} + 2 c_1^2 \frac{\tilde{p}^2 \tgam^2}{n \lambda^2} \right),\\
\tsig^2 \mathrm{tr} \{ (\TX \TX^T + n \lambda I)^{-2} \TX \varSigma \TX^T \} &\ge \tsig^2 \frac{1}{18 c_1^2 c_2 c_3^2} \left( \frac{k^*}{n} \min \{ 1, \frac{1}{\lambda^2} \} + \frac{\tgam^2 (\tilde{p} - k^*)}{n \lambda^2} \right).    
\end{aligned}  
\end{equation}

\subsection{Analysis on $\mathcal{L}_{\mathrm{ft}}$}

From Condition~\ref{cond:order}, we have the following results:
\begin{align*}
 & O \left(\max\{ \gamma, \tgam \} + n^{- (1 - \xi) / 2} \right) \ll O(\zeta_2 \tilde{p} \tgam) , \quad O \left( \zeta_1 (1 + n \tilde{p} \tgam / p)  \right) \ll O(\zeta_2 \tilde{p} \tgam), \\
 & O \left( \sigma^2 (n^{-1}, n \tilde{p} \tgam / (p^2 \gamma)) \right) \ll O(\zeta_2 \tilde{p} \tgam),\\
 & O\left(\zeta_2 (1 + n \tgam ) \right) \asymp O(\zeta_2 \tilde{p} \tgam), \quad O \left(\tsig^2 (n^{-1} + n / \tilde{p}) \right) \asymp O(\zeta_2 \tilde{p} \tgam).
\end{align*}
Then comparing \eqref{eq:theta_c}, \eqref{eq:zeta_1}, \eqref{eq:sig}, \eqref{eq:zeta2_s}, \eqref{eq:zeta2_l}, \eqref{eq:tsig_s} and \eqref{eq:tsig_l}, 
we could obtain that the excess risks on different estimators always dominated by terms related to $\zeta_2$ and $\tsig^2$. So for the ridge regression, we have
\begin{equation*}
\mathcal{L}_{\mathrm{ft}}(\hat{\theta}_\lambda) \approx \zeta_2 \mathrm{tr}\{ (I - \TX^T(\TX \TX^T + n \lambda I)^{-1} \TX)^2 \tvar \} + \tsig^2 \mathrm{tr}\{ (\TX \TX^T + n \lambda I)^{-2} \TX \tvar \TX^T\} := f(\lambda), 
\end{equation*}
then take derivative with respect to $\lambda$, we could further obtain that
\begin{equation*}
    f'(\lambda) = 2n (\zeta_2 n \lambda - \tsig^2) \mathrm{tr} \{ (\TX \TX^T + n \lambda I)^{-3} \TX \tvar \TX^T \},
\end{equation*}
which implies that the optimal choice of $\lambda$ nearly equals to 
\begin{equation*}
  \lambda' =  \frac{\tsig^2}{n \zeta_2} = O(n^{-1}).
\end{equation*}
As $\hat{\theta}_\lambda = \hat{\theta}_2$ if $\lambda = 0$ and the optinal value $\lambda' > 0$, we have
\begin{equation}\label{eq:re1}
   \mathcal{L}_{\mathrm{ft}}(\hat{\theta}_\lambda) \le \mathcal{L}_{\mathrm{ft}}(\hat{\theta}_2), \quad \forall 0 \le \lambda \le 2\lambda'.
\end{equation}
And according to Lemma~\ref{lem_eigen}, considering Condition~\ref{cond:order}, with probability at least $1 - 2 e^{- n / c}$, we have
\begin{equation*}
    \mu_n(\TX\TX^T) \ge \frac{1}{c_1} (\tilde{p} - k^*) \tgam > \frac{\tsig^2}{\zeta_2},
\end{equation*}
which implies that
\begin{equation*}
    \tsig^2 \mathrm{tr} \{ (\TX\TX^T)^{-2} \TX\tvar \TX^T \} - \zeta_2 \mathrm{tr} \{ (\TX \TX^T)^{-1} \TX \tvar \TX^T \} = - \mathrm{tr} \{ (\TX \TX^T)^{-2} ( \zeta_2 \TX \TX^T - \tsig^2 I ) \TX \tvar \TX^T \} < 0,
\end{equation*}
so we could obtain that
\begin{equation}\label{eq:re2}
  \mathcal{L}_{\mathrm{ft}}(\hat{\theta}_2) <  \mathcal{L}_{\mathrm{ft}}(\hat{\theta}_1). 
\end{equation}
Combing the results in \eqref{eq:re1} and \eqref{eq:re2}, we could finish the proof of the first item in Theorem~\ref{thm}.

Finally, while choosing $0 \le \lambda < \lambda'$, according to \eqref{eq:zeta2_s}, \eqref{eq:zeta2_l}, \eqref{eq:tsig_s} and \eqref{eq:tsig_l}, we have
\begin{equation*}
    \tsig^2 \mathrm{tr} \{ (\TX \TX^T + n \lambda I)^{-2} \TX \tvar \TX^T \} \asymp \zeta_2 \mathrm{tr} \{ (\TX \TX^T + n \lambda I)^{-1} \TX \tvar \TX^T \} \asymp \zeta_2 \mathrm{tr}\{ \tvar \}.
\end{equation*}
And for the excess risk of ensemble estimator, we have
\begin{align*}
 \mathcal{L}_{\mathrm{ft}}(\hat{\theta}^{\tau}_{\lambda}) &\approx  \zeta_2 \mathrm{tr}\{ (I - \tau \TX^T(\TX \TX^T + n \lambda I)^{-1} \TX)^2 \tvar \}  + \tau^2 \tsig^2 \mathrm{tr}\{ (\TX \TX^T + n \lambda I)^{-2} \TX \tvar \TX^T\} := g(\tau),
\end{align*}
then take derivative with respect to $\tau$, we could obtain that
\begin{align*}
     g'(\tau) &= 2 \tau \zeta_2 \mathrm{tr} \{ (\TX \TX^T + n \lambda I)^{-1} \TX \TX^T (\TX \TX^T + n \lambda I)^{-1} \TX \tvar \TX^T \} + 2 \tau \tsig^2 \mathrm{tr}\{ (\TX \TX^T + n \lambda I)^{-2} \TX \tvar \TX^T\}\\
     & \quad - 2 \zeta_2 \mathrm{tr}\{ (\TX \TX^T + n \lambda I)^{-1} \TX \tvar \TX^T\},
\end{align*}
which implies the optimal choice of $\tau$ is as
\begin{equation*}
  \tau'(\lambda) =  \frac{\zeta_2 \mathrm{tr}\{ (\TX \TX^T + n \lambda I)^{-1} \TX \tvar \TX^T\}}{\zeta_2 \mathrm{tr} \{ (\TX \TX^T + n \lambda I)^{-1} \TX \TX^T (\TX \TX^T + n \lambda I)^{-1} \TX \tvar \TX^T \} + \tsig^2 \mathrm{tr}\{ (\TX \TX^T + n \lambda I)^{-2} \TX \tvar \TX^T\}}.
\end{equation*}
While $0 \le \lambda < \lambda' = \tsig^2 / (n \zeta_2)$, we can obtain that $0 \le \tau'(\lambda) \le 1$, due to the fact that
\begin{align*}
& \quad \zeta_2 \mathrm{tr} \{ (\TX \TX^T + n \lambda I)^{-1} \TX \TX^T (\TX \TX^T + n \lambda I)^{-1} \TX \tvar \TX^T \} + \tsig^2 \mathrm{tr}\{ (\TX \TX^T + n \lambda I)^{-2} \TX \tvar \TX^T\} \\
& \quad - \zeta_2 \mathrm{tr}\{ (\TX \TX^T + n \lambda I)^{-1} \TX \tvar \TX^T\} \\
&= \mathrm{tr} \left\{ (\TX\TX^T + n \lambda I)^{-2} \left( \tsig^2 I + \zeta_2 (\TX \TX^T + n \lambda I) \TX \TX^T (\TX \TX^T + n \lambda I)^{-1} - \zeta_2 (\TX \TX^T + n \lambda I) \right) \TX \tvar \TX^T \right\}\\
&= \mathrm{tr} \left\{ (\TX\TX^T + n \lambda I)^{-2} \left( \tsig^2 I + \zeta_2  \TX \TX^T  - \zeta_2 (\TX \TX^T + n \lambda I) \right) \TX \tvar \TX^T \right\}\\
&= (\tsig^2 - \zeta_2 n \lambda)  \mathrm{tr} \{ (\TX \TX^T + n \lambda I)^{-2} \TX \tvar \TX^T \} \ge 0.
\end{align*}
So we could draw the conclusion that for any $\tau'(\lambda) \le \tau < 1$, we have
\begin{equation}\label{eq:re3}
     \mathcal{L}_{\mathrm{ft}}(\hat{\theta}_\lambda^{\tau}) <  \mathcal{L}_{\mathrm{ft}}(\hat{\theta}_\lambda),
\end{equation}
which finishes the proof of the third item in Theorem~\ref{thm}.

\subsection{Analysis on $\mathcal{L}_{\mathrm{pre}} + \mathcal{L}_{\mathrm{ft}}$}

Similar to the analysis on $\mathcal{L}_{\mathrm{ft}}$, based on Condition~\ref{cond:order}, we could compare \eqref{eq:theta_c}, \eqref{eqs:theta_c}, \eqref{eq:zeta_1}, \eqref{eqs:zeta_1}, \eqref{eq:sig}, \eqref{eqs:sig}, \eqref{eq:zeta2_s}, \eqref{eq:zeta2_l}, \eqref{eqs:zeta2}, \eqref{eq:tsig_s}, \eqref{eq:tsig_l}, \eqref{eqs:tsig_s} and \eqref{eqs:tsig_l}, and obtain that
\begin{align*}
\mathcal{L}_{\mathrm{pre}}(\hat{\theta}_\lambda) + \mathcal{L}_{\mathrm{ft}}(\hat{\theta}_\lambda) &\approx \zeta_2 \mathrm{tr}\{ (I - \TX^T(\TX \TX^T + n \lambda I)^{-1} \TX)^2 \tvar \} + \tsig^2 \mathrm{tr}\{ (\TX \TX^T + n \lambda I)^{-2} \TX \tvar \TX^T\} \\
& \quad + \zeta_2 \mathrm{tr} \{ [\TX^T (\TX \TX^T + n \lambda I)^{-1} \TX]^2 \varSigma \} + \tsig^2 \mathrm{tr}\{ (\TX \TX^T + n \lambda I)^{-2} \TX \varSigma \TX^T \} \\
&= \zeta_2 \mathrm{tr}\{ (I - \TX^T(\TX \TX^T + n \lambda I)^{-1} \TX)^2 \tvar \} + 2 \tsig^2 \mathrm{tr}\{ (\TX \TX^T + n \lambda I)^{-2} \TX \tvar \TX^T\} \\
& \quad + \zeta_2 \mathrm{tr} \{ [\TX^T (\TX \TX^T + n \lambda I)^{-1} \TX]^2 \tvar \} \\
&:= h(\lambda),
\end{align*}
where the second equality is from \eqref{eq:eq1} and \eqref{eq:eq2}. Taking derivative with respect to $\lambda$, we have
\begin{equation*}
    h'(\lambda) = 2n \mathrm{tr} \{ (\TX \TX^T + n \lambda I)^{-3} [( \zeta_2 n \lambda - 2 \tsig^2)I - \zeta_2 \TX \TX^T] \TX \tvar \TX^T \},
\end{equation*}
which implies that while $\lambda \le 2 \lambda' = 2 \tsig^2 / (n \zeta_2)$, we could always obtain
\begin{equation*}
    h'(\lambda) < 0.
\end{equation*}
So for any $0 < \lambda \le 2 \lambda'$, we have
\begin{equation}\label{eqs:re1}
   \mathcal{L}_{\mathrm{pre}}(\hat{\theta}_\lambda) + \mathcal{L}_{\mathrm{ft}}(\hat{\theta}_\lambda)  < \mathcal{L}_{\mathrm{pre}}(\hat{\theta}_2) + \mathcal{L}_{\mathrm{ft}}(\hat{\theta}_2). 
\end{equation}
And with this range of $\lambda$, we could obtain the similar result as:
\begin{align*}
\mathcal{L}_{\mathrm{pre}}(\hat{\theta}^{\tau}_{\lambda}) + \mathcal{L}_{\mathrm{ft}}(\hat{\theta}^{\tau}_{\lambda}) &\approx  \zeta_2 \mathrm{tr}\{ (I - \tau \TX^T(\TX \TX^T + n \lambda I)^{-1} \TX)^2 \tvar \}  + 2 \tau^2 \tsig^2 \mathrm{tr}\{ (\TX \TX^T + n \lambda I)^{-2} \TX \tvar \TX^T\} \\
& \quad + \tau^2 \zeta_2 \mathrm{tr} \{ (\TX \TX^T + n \lambda I)^{-2} \TX \TX^T \TX \tvar \TX^T \}\\
&:= J(\tau),
\end{align*}
and taking derivative with respect to $\tau$, we have
\begin{align*}
 J'(\tau) &= 4 \tau \zeta_2 \mathrm{tr} \{ (\TX \TX^T + n \lambda I)^{-2} \TX \TX^T \TX \tvar \TX^T \} + 4 \tau \tsig^3 \mathrm{tr} \{ (\TX \TX^T + n \lambda I)^{-2} \TX \tvar \TX^T \}\\
 & \quad - 2 \zeta_2 \mathrm{tr}\{ (\TX \TX^T + n \lambda I)^{-1} \TX \tvar \TX^T \},
\end{align*}
which implies that the optimal choice of $\tau$ is $\tau'(\lambda) / 2$. And while $\lambda \le 2 \lambda'$, we have
\begin{align*}
 & \quad 4 \tau \zeta_2 \mathrm{tr} \{ (\TX \TX^T + n \lambda I)^{-2} \TX \TX^T \TX \tvar \TX^T \} + 4 \tau \tsig^3 \mathrm{tr} \{ (\TX \TX^T + n \lambda I)^{-2} \TX \tvar \TX^T \}\\
 & \quad - 2 \zeta_2 \mathrm{tr}\{ (\TX \TX^T + n \lambda I)^{-1} \TX \tvar \TX^T \} \\
 &= 2 \mathrm{tr}\{ (\TX \TX^T + n \lambda I)^{-2} [\zeta_2 \TX \TX^T + (2 \tsig^2 - \zeta_2 n \lambda)I] \TX \tvar \TX^T \}> 0, 
\end{align*}
which implies that $0 \le \tau'(\lambda) / 2 \le 1$. So we could further obtain that for any $\tau'(\lambda) / 2 \le \tau < 1$, we have
\begin{equation}\label{eqs:re2}
   \mathcal{L}_{\mathrm{pre}}(\hat{\theta}^{\tau}_\lambda) + \mathcal{L}_{\mathrm{ft}}(\hat{\theta}^{\tau}_\lambda) <  \mathcal{L}_{\mathrm{pre}}(\hat{\theta}_{\lambda}) + \mathcal{L}_{\mathrm{ft}}(\hat{\theta}_{\lambda}).
\end{equation}
Combing the results in \eqref{eqs:re1} and \eqref{eqs:re2}, we could finish the proof of the second item in Theorem~\ref{thm}.

\section{Auxiliary Lemmas}

\begin{lemma}[Lemma 10 in \citealp{bartlett2020benign}]\label{lem_eigen}
 There are constants $b, c \ge 1$ such that, for any $k \ge 0$, with probability at least $1 - 2 e^{- \frac{n}{c}}$,
\begin{enumerate}
\item for all $i \ge 1$,
\begin{equation*}
    \mu_{k+1}(A_{-i}) \le \mu_{k+1}(A) \le \mu_{1}(A_{k}) \le c_1 (\sum_{j > k} \lambda_j + \lambda_{k+1} n);
\end{equation*}
\item for all $1 \le i \le k$,
\begin{equation*}
    \mu_{n}(A) \ge \mu_{n}(A_{-i}) \ge \mu_{n}(A_{k}) \ge \frac{1}{c_1} \sum_{j > k} \lambda_j - c_1 \lambda_{k+1} n;
\end{equation*}
\item if $r_k \ge bn$, then
\begin{equation*}
    \frac{1}{c_1} \lambda_{k+1} r_k \le \mu_n (A_k) \le \mu_1 (A_k) \le c_1 \lambda_{k+1} r_k ,
\end{equation*}
\end{enumerate}
where $c_1 > 1$ is a constant only depending on $\sigma_x$.
\end{lemma}

\begin{lemma}[Corollary 24 in \citealp{bartlett2020benign}]\label{lem_subspacenorm}
 For any centered random vector $\bm{z} \in \mathbb{R}^n$ with independent $\sigma^2_x$ sub-Gaussian coordinates with unit variances, any $k$ dimensional random subspace $\mathscr{L}$ of $\mathbb{R}^n$ that is independent of $\bm{z}$, and any $t > 0$, with probability at least $1 - 3 e^{-t}$,
\begin{equation*}
\begin{aligned}
& \parallel z \parallel^2 \le n + 2 (162e)^2 \sigma_x^2(t + \sqrt{nt}),\\
& \parallel \Pi_{\mathscr{L}} z \parallel^2 \ge n - 2 (162e)^2 \sigma_x^2 (k + t + \sqrt{nt}),
\end{aligned}
\end{equation*}
where $\Pi_{\mathscr{L}}$ is the orthogonal projection on $\mathscr{L}$.
\end{lemma}

\begin{lemma}\label{lem_ridgeeigen}
 There are constants $b, c \ge 1$ such that, for any $k \ge 0$, with probability at least $1 - 2 e^{- \frac{n}{c}}$:
 \begin{enumerate}
\item for all $i \ge 1$,
\begin{equation*}
    \mu_{k+1}(A_{-i} + \lambda  I) \le \mu_{k+1}(A + \lambda  I) \le \mu_{1}(A_{k} + \lambda  I) \le c_1 (\sum_{j > k} \lambda_j + \lambda_{k+1} n) + \lambda ;
\end{equation*}
\item for all $1 \le i \le k$,
\begin{equation*}
    \mu_{n}(A + \lambda  I) \ge \mu_{n}(A_{-i} + \lambda  I) \ge \mu_{n}(A_{k} + \lambda  I) \ge \frac{1}{c_1} \sum_{j > k} \lambda_j - c_1 \lambda_{k+1} n + \lambda ;
\end{equation*}
\item if $r_k \ge bn$, then
\begin{equation*}
    \frac{1}{c_1}  \lambda_{k+1} r_k + \lambda \le \mu_n (A_k + \lambda  I) \le \mu_1 (A_k + \lambda  I) \le c_1 \lambda_{k+1} r_k + \lambda.
\end{equation*}   
\end{enumerate}
\end{lemma}

\begin{proof}
With Lemma \ref{lem_eigen}, the first two claims follow immediately. For the third claim: if $r_k \ge bn$, we have that $bn \lambda_{k+1} \le \sum_{j > k} \lambda_j$, so
\begin{equation*}
 \begin{aligned}
& \mu_1 (A_k + \lambda  I) \le c_1 \lambda_{k+1} r_k (\Sigma) + \lambda  \le  \lambda  + c_1 \lambda_{k+1} r_k, \\
& \mu_n (A_k + \lambda  I) \ge \frac{1}{c_1} \lambda_{k+1} r_k  + \lambda  \ge \frac{1}{c_1}  \lambda_{k+1} r_k (\Sigma) + \lambda,
 \end{aligned}   
\end{equation*}
for the same constant $c_1 > 1$ as in Lemma \ref{lem_eigen}.
\end{proof}

\begin{lemma}[Proposition 2.7.1 in \citealp{vershynin2018high}]\label{lem_sg_se}
 For any random variable $\xi$ that is centered, $\sigma^2$-subgaussian, and unit variance, $\xi^2 - 1$ is a centered $162e\sigma^2$-subexponential random variable, that is,
\begin{equation*}
    \mathbb{E}\exp(\lambda (\xi^2 - 1)) \le \exp((162e\lambda \sigma^2)^2), 
\end{equation*}
for all such $\lambda$ that $| \lambda | \le 1 / (162 e \sigma^2)$.
\end{lemma}

\begin{lemma}[Lemma 15 in \citealp{bartlett2020benign}]\label{lem_sum}
Suppose that $\{ \eta_i \}$ is a sequence of non-negative random variables, and that $\{ t_i \}$ is a sequence of non-negative real numbers (at least one of which is strictly positive) such that, for some $\delta \in (0, 1)$ and any $i \ge 1$, $\Pr(\eta_i > t_i) \ge 1 - \delta$. Then,
\begin{equation*}
    \Pr\left(\sum_i \eta_i \ge \frac{1}{2} \sum_i t_i\right) \ge 1 - 2 \delta.
\end{equation*}
\end{lemma}

\begin{lemma}[Lemma 2.7.6 in \citealp{vershynin2018high}]\label{lem_stnorm}
 For any non-increasing sequence $\{ \lambda_i \}_{i=1}^{\infty}$ of non-negative numbers such that $\sum_i \lambda_i < \infty$, and any independent, centered, $\sigma-$subexponential random variables $\{ \xi_i \}_{i=1}^{\infty}$, and any $x > 0$, with probability at least $1 - 2 e^{-x}$
\begin{equation*}
    | \sum_i\lambda_i \xi_i | \le 2 \sigma \max \left( x \lambda_1, \sqrt{x \sum_i \lambda_i^2} \right).
\end{equation*}
\end{lemma}

\begin{lemma}[Theorem 9 in \citet{koltchinskii2017concentration}]\label{lem_eigenx}
Let $z_1, \dots, z_n$ be i.i.d. sub-gaussian random variables with zero mean, then with probability at least $1 - 2 e^{- t}$,
 \begin{equation*}
     \| \mathbb{E} z z^T - \frac{1}{n} \sum_{i=1}^n z_i z_i^T \|_2 \le  \| \mathbb{E} zz^T \|_2 \max \{ \sqrt{\frac{\text{trace}(\mathbb{E} zz^T)}{n}}, \frac{\text{trace}(\mathbb{E} zz^T)}{n}, \sqrt{\frac{t}{n}}, \frac{t}{n} \}.
\end{equation*}
\end{lemma}

\end{document}